\documentclass{article}

\usepackage[preprint]{neurips_2023}

\usepackage[utf8]{inputenc}
\usepackage[T1]{fontenc}
\usepackage[pagebackref]{hyperref}
\usepackage{url}
\usepackage{booktabs}
\usepackage{amsfonts}
\usepackage{nicefrac}
\usepackage{microtype}
\usepackage{xcolor}
\usepackage{amsthm,amssymb}
\usepackage{graphicx}

\renewcommand*{\backrefalt}[4]{%
    \ifcase #1 \footnotesize{(Not cited.)}%
    \or        \footnotesize{(Cited on page~#2)}%
    \else      \footnotesize{(Cited on pages~#2)}%
    \fi}

\newtheorem{asm}{Assumption}
\newtheorem{thm}{Theorem}
\newtheorem{lem}{Lemma}

\newtheorem{defn}{Definition}
\newtheorem{prop}{Proposition}
\newtheorem{corollary}{Corollary}

\newcommand{\diag}{\mathrm{diag}}
\newcommand{\mA}{\mathbf{A}}

\usepackage{algorithm}
\usepackage{algpseudocode}
\usepackage{amsmath}

\title{Prodigy: An Expeditiously Adaptive\\ Parameter-Free Learner}

\author{%
  Konstantin Mishchenko\\
  Samsung AI Center\\
  \\
    \And
  Aaron Defazio \\
  Meta AI, Fundamental AI Research (FAIR) team \\
}

\begin{document}

\maketitle

\begin{abstract}
We consider the problem of estimating the learning rate in adaptive methods, such as AdaGrad and Adam. We propose Prodigy, an algorithm that provably estimates the distance to the solution $D$, which is needed to set the learning rate optimally. At its core, Prodigy is a modification of the D-Adaptation method for learning-rate-free learning. It improves upon the convergence rate of D-Adaptation by a factor of $\mathcal{O}(\sqrt{\log(D/d_0)})$, where $d_0$ is the initial estimate of $D$. We test Prodigy on 12 common logistic-regression benchmark datasets, VGG11 and ResNet-50 training on CIFAR10, ViT training on Imagenet, LSTM training on IWSLT14, DLRM training on Criteo dataset, VarNet on Knee MRI dataset, as well as RoBERTa and GPT transformer training on BookWiki. Our experimental results show that our approach consistently outperforms D-Adaptation and reaches test accuracy values close to that of hand-tuned Adam.
\end{abstract}

\section{Introduction}

Optimization is an essential tool in modern machine learning, enabling efficient solutions to large-scale problems that arise in various domains, such as computer vision, natural language processing, and reinforcement learning. One of the key challenges is the selection of appropriate learning rates, which can significantly impact the convergence speed and the quality of the final solution. Learning-rate tuning has been particularly challenging in applications where there are multiple agents that use their own optimizer. For instance, when training Generative Adversarial Networks (GANs)~\citep{goodfellow2020generative}, there are two neural networks with different architectures. In federated learning, tuning is even more challenging~\citep{khodak2021federated}, since there might be billions of devices~\citep{kairouz2019advances}, each optimizing their objective locally. Another example is Neural Architecture Search (NAS)~\citep{zoph2017neural}, where the goal is to find the best neural network architecture automatically by training a lot of networks and evaluating them on a validation set. In such cases, it becomes very expensive to manually tune the learning rate.

Recently, \emph{parameter-free} adaptive learning rate methods \citep{coin-betting, pmlr-v75-cutkosky18a, pdecoin, parameterfreesgd, dog} have gained considerable attention due to their ability to automatically adjust learning rates based on the problem structure and data characteristics. Among these, the D-Adaptation method, introduced by \citet{defazaio2023learning}, has emerged as a promising practical approach for learning-rate-free optimization.

Given a convex objective function $f$, D-Adaptation works by maintaining a lower bound on the initial distance to solution $D=\left\Vert x_{0}-x_{*}\right\Vert$, for any $x_*$ in the solution set of the following problem: 
\[
    \min_{x\in\mathbb{R}^p} f(x).
\]
In practice, the lower bound estimated by D-Adaptation increases rapidly during the course of optimization, plateauing to a value close to the true $D$.
This $D$ quantity is the key unknown constant needed to set the learning rate for non-smooth optimization methods, forming the numerator of the step size:
\[
\gamma_{k+1}=\frac{D}{\sqrt{\sum_{i=0}^{k}\left\Vert g_{i}\right\Vert ^{2}}},\qquad \text{where}\ D=\|x_0-x_*\|,
\]
and the denominator is based on the AdaGrad step size~\cite{adagrad,lessregret,ward2019adagrad}. The Gradient Descent form of D-Adaptation simply plugs in the current lower bound at each step in place of $D$. This simple approach can be applied to estimate the step size in Adam~\citep{kingma2015adam}, which yields state-of-the-art performance across a wide-range of deep learning problems. \citet{defazaio2023learning} also show that asymptotically, D-Adaptation is as fast as specifying the step size using the true $D$ (up to small constant factors). 

\subsection*{Contributions}
We summarize our contributions as follows.
\begin{itemize}
    \item We present \emph{Prodigy}, a modification of D-Adaptation that improves it's worst-case non-asymptotic convergence rate.
    \item Through extensive experiments, we demonstrate that Prodigy establishes a new state-of-the-art for learning rate adaptation, outperforming D-Adaptation.
    \item We develop a lower complexity bound for methods which grow the learning rate at most exponentially fast. We show that this covers all methods that avoid significant overshooting.
    \item Prodigy is highly practical. Our open-source implementation is already widely used for fine-tuning of vision and language models, and is the recommended optimizer for \emph{Hugging Face Diffusers} DreamBooth LoRA training.
\end{itemize}

\section{Prodigy Approach}
To understand how we can improve upon D-Adaptation, let us take a closer look at some details of its analysis. In D-Adapted Dual Averaging, the gradient at iteration $k$ is scaled with weight $\lambda_k$. This leads to the error term:
\[
    \text{D-adaptation\ error}=\sum_{k=0}^n \lambda_k^2\gamma_k\|g_k\|^2.
\]
The theory then proceeds to upper bound this sum using the largest of all $\lambda_k$'s by using the upper bound $\lambda_k\le \lambda_n$. This, however, is quite pessimistic since then $\lambda_k$ is set to be $\lambda_k = d_k$, so $\lambda_n$ can be as large as $D$ and $\lambda_k$ can be as small as $d_0$. Therefore, replacing $\lambda_k^2$ with $\lambda_n^2$ can introduce a multiplicative error of $\frac{D^2}{d_0^2}$ in this term. 

\begin{figure}[t]
\begin{minipage}[t]{0.47\textwidth}
\begin{algorithm}[H]
\begin{algorithmic}[1]
    \State {\bfseries Input:} $d_0>0$, $x_0$,  $G\ge 0$
    \For{$k=0$ {\bfseries to} $n$}
    
    \State $g_{k} \in \partial f(x_{k})$
        \State Choose weight $\lambda_k$ (default: $\lambda_k=1$)
        \State  $\eta_{k}=\dfrac{d_k^2 \lambda_{k}}{\sqrt{d_k^2 G^2 + \sum_{i=0}^{k}d_i^2\lambda_i^2\left\Vert g_{i}\right\Vert ^{2}}}$
	\State $x_{k+1}=x_k-\eta_{k}g_k$
	\State $\hat{d}_{k+1}=\dfrac{\sum_{i=0}^k \eta_i\langle g_i, x_0 - x_i\rangle}{\left\Vert x_{k+1} - x_0\right\Vert }$
        \State $d_{k+1}= \max(d_k, \hat d_{k+1})$
    \EndFor
	\State Return $\hat{x}_{n}=\frac{1}{n+1}\sum ^n_{k=0} \eta_k x_k$
\end{algorithmic}
\caption{\label{alg:dadagradv2gd}Prodigy (GD version)}
\end{algorithm}
\end{minipage}
\hfill
\begin{minipage}[t]{0.47\textwidth}
\begin{algorithm}[H]
\begin{algorithmic}[1]
    \State {\bfseries Input:} $d_0>0$, $x_0$,  $G\ge 0$; $s_{0} = 0\in \mathbb{R}^p$
    \For{$k=0$ {\bfseries to} $n$}
    
    \State $g_{k} \in \partial f(x_{k})$
        \State  $\lambda_k = d_k^2$
        \State $s_{k+1} = s_k + \lambda_k g_k$
	\State $\hat{d}_{k+1}=\dfrac{\sum_{i=0}^k \lambda_i\langle g_i, x_0 - x_i\rangle}{\left\Vert s_{k+1}\right\Vert }$
        \State $d_{k+1}= \max(d_k, \hat d_{k+1})$
        \State  $\gamma_{k+1}=\dfrac{1}{\sqrt{\lambda_{k+1}G^2 + \sum_{i=0}^{k}\lambda_i\left\Vert g_{i}\right\Vert ^{2}}}$
	\State $x_{k+1}=x_0-\gamma_{k+1}s_{k+1}$
    \EndFor
	\State Return $\bar{x}_{n}=\frac{1}{n+1}\sum ^n_{k=0}\lambda_k x_k$
\end{algorithmic}
\caption{\label{alg:dadagradv2da}Prodigy (Dual Averaging version)}
\end{algorithm}
\end{minipage}
\end{figure}

We take a different approach and instead handle the error term using modified AdaGrad-like step sizes. In the AdaGrad theory, the error term does not have any $\lambda_k^2$ factors, which is exactly why AdaGrad places $\sqrt{\sum_{i=0}^k\|g_i\|^2}$ in the step-size denominator. The required modification is thus clear: since the error terms are now $d_i^2\|g_i\|^2$ instead of $\|g_i\|^2$, the new adaptive step size should be
\[
    \gamma_{k+1} = \frac{1}{\sqrt{\sum_{i=0}^k d_i^2\|g_i\|^2}}
\]
for the Dual Averaging algorithm and
\[
    \eta_k = \frac{d_k^2}{\sqrt{\sum_{i=0}^k d_i^2\|g_i\|^2}}
\]
for the Gradient Descent algorithm. This way, we can still control the error term of D-Adaptation but the obtained step size is provably larger since $d_k$ is non-decreasing. For instance, for Gradient Descent, we have
\[
    \frac{d_k^2}{\sqrt{\sum_{i=0}^k d_i^2\|g_i\|^2}}
    \ge \frac{d_k}{\sqrt{\sum_{i=0}^k \|g_i\|^2}}.
\]
Having larger step sizes while preserving the main error term is the key reason why the new algorithms converge, as we show below, with a faster rate.

Notice, however, that the methods might still be slow because the denominator in the step size might grow too large over time. To remedy this, we introduce a modification for the Gradient Descent step size by placing an extra weight $\lambda_k$ next to the gradients:
\[
    \eta_{k}
    =\frac{d_k^2 \lambda_{k}}{\sqrt{ \sum_{i=0}^{k}d_i^2\lambda_i^2\left\Vert g_{i}\right\Vert ^{2}}}.
\]
In fact, the modified step size might even increase between iterations, whereas the AdaGrad step size always decreases. We will show that as long as $\lambda_k$ does not grow too quickly, the worst-case convergence rate is almost the same.

To establish non-asymptotic theory, we also introduce in our algorithms an extra term $G^2$ in the denominator which upper bound the gradient norm. We define it formally in the assumption below.
\begin{asm}
    We assume that the objective $f$ is $G$-Lipschitz, which implies that its gradients are bounded by $G$: for any $x\in \mathbb{R}^p$ and $g\in \partial f(x)$, it holds $\|g\|\le G$.
\end{asm}

Algorithm~\ref{alg:dadagradv2gd} and Algorithm~\ref{alg:dadagradv2da} give Gradient Descent and the Dual Averaging variants of our new method. In contrast to AdaGrad, they estimate the \emph{pro}duct of $D$ and $G$ in the denominator, so we call the proposed technique \emph{Prodigy}. We give the following convergence result for Algorithm~\ref{alg:dadagradv2gd}:
\begin{thm}\label{thm:gd}
    Assume $f$ is convex and $G$-Lipschitz. Given any weights $1\le\lambda_0\le\dotsb \le\lambda_n$, the functional gap of the average iterate of Algorithm~\ref{alg:dadagradv2gd} converges as
    \begin{equation}
        f(\hat x_n) - f_*
        \le \sqrt{2\lambda_n}DG\frac{d_{n+1}(2 + \log(1+\sum_{k=0}^n \lambda_k^2))}{\sqrt{\sum_{k=0}^n\lambda_k d_k^2}}, \label{eq:gd_main_bound}
    \end{equation}
    where $\hat x_n = \frac{1}{n+1}\sum_{k=0}^n \eta_k x_k$ is the weighted average iterate.
\end{thm}
Notice that we have the freedom to choose any non-decreasing sequence $\lambda_k$ as long as the right-hand side is decreasing. This allows us to put much more weight on the recent gradients and get more reasonable step sizes. For instance, we can choose $\lambda_k = k^p$, where $p>0$ and since $\sum_{k=0}^k k^{2p}=\mathcal{O}(k^{2p+2})$, it would result in an extra factor of $1+p$ in the numerator due to the log term. The denominator, on the other hand, would increase as well, giving us a trade-off that depends on the values of $d_k$'s. We note that weights $\lambda_k=k$ have appeared previously in Accelegrad~ \citep{levy2018online} and UniXGrad~\citep{kavis2019unixgrad}, which combine AdaGrad step sizes with momentum, and $\lambda_k=\sqrt{k}$ weighting is used in the recent MADGRAD method \citep{defazio2022adaptivity}.

To understand why this improves the convergence rate, consider the following lemma, which we prove in the appendix. The lemma presents an upper bound on the terms related to the $d_k$ sequence in the right-hand side of equation~\ref{eq:gd_main_bound}.
\begin{lem}\label{lem:d_sequence} Let $d_{0}\le d_{1}\le\dotsb\le d_{N}$ be positive
numbers and assume $N\ge2\log_{2+}(\frac{d_{N}}{d_{0}})$, where $\log_{2+}(\cdot) = 1 + \log_2(\cdot)$. Then, 
\[
\min_{t<N}\frac{d_{t+1}}{\sqrt{\sum_{k=0}^{t}d_{k}^{2}}}\le\frac{4\sqrt{\log_{2+}\bigl(\frac{d_{N}}{d_{0}}\bigr)}}{\sqrt{N}}.
\]
\end{lem}
In contrast to the bound in \cite{defazaio2023learning}, we bound $\frac{d_{t+1}}{\sqrt{\sum_{k=0}^{t}d_{k}^{2}}}$ instead of $\frac{d_{t+1}}{\sum_{k=0}^{t}d_{k}}$. This is the reason why the overall guarantee improves by a factor of $\sqrt{\log_2(D/d_0)}$. For instance, if we set $\lambda_k=1$ for all $k$ and substitute the bound from Lemma~\ref{lem:d_sequence}, we get the convergence rate
\[
    f(\hat x_t) - f_*
    = \mathcal{O} \left( \frac{GD\log(n+1)\sqrt{\log_{2+}(D/d_0)}}{\sqrt{n}} \right).
\]
where $t\le n$ is chosen as the argmin from Lemma~\ref{lem:d_sequence}. Furthermore, for arbitrary increasing positive weights, we get the following guarantee by applying Lemma~\ref{lem:d_sequence} directly to the bound in Theorem~\ref{thm:gd}:
\[
    f(\hat x_t) - f_*
    =\mathcal{O} \left( \frac{GD\log(n+1)\sqrt{\log_{2+}(\lambda_{n+1}D/d_0)}}{\sqrt{n}}\log\left(\sum_{k=0}^n\lambda_k^2\right)\right).
\]
Even though our theory does not guarantee that it is beneficial to use increasing weights $\lambda_k$, this result is, to the best of our knowledge, new for AdaGrad-like methods. It allows for a wide range of choices in $\lambda_k$. For example, if we set $\lambda_k = \beta_2^{-k^p}$ with $\beta_2<1$ and $p<1/3$, then the method is still guaranteed to converge at the rate of $\mathcal{O}\left(\frac{1}{n^{(1-3p)/2}}\right)$. This is of particular interest when we study Adam-like methods, see Section~\ref{sec:adam} for a discussion.

The logarithmic term $\log(n+1)$ is, however, not necessary and only arises due to the use of Gradient Descent update. The Dual Averaging update, provides a tighter guarantee as given in the next theorem.
\begin{thm}\label{thm:da}
Let $f$ be a convex and $G$-Lipschitz function. For Algorithm~\ref{alg:dadagradv2da}, it holds that:
\[
        f(\overline x_t) - f_*
        \le \frac{4GD}{\sqrt{n}}\sqrt{\log_{2+}\Bigl(\frac{D}{d_0}\Bigr)},
\]
where $t=\arg\min_{k\le n} \frac{d_{k+1}}{\sqrt{\sum_{i=0}^k d_i^2}}$ and $\log_{2+}(\cdot) = 1 + \log_2(\cdot)$.
\end{thm}
Comparing this with the previous rate, the only difference is the removal of a  multiplicative $\log(n+1)$ factor. This improvement, however, is mostly theoretical as Gradient Descent typically performs better in practice than Dual Averaging. We also note that we do not have a convergence result for Algorithm~\ref{alg:dadagradv2da} with weights other than $\lambda_k=d_k^2$. This is due to the fact that the DA analysis requires the step size to be monotonically decreasing, so we cannot place an extra weighting factor in the numerator of $\gamma_{k+1}$.

\section{Lower Complexity Bounds for Exponentially Bounded Algorithms}
We can obtain an interesting class of algorithms, which contains our two D-Adaptation variants, by restricting the rate of growth.
\begin{defn}
An optimization algorithm is exponentially bounded if there exists a constant $d_0$, so that for any sequence of G-bounded gradients it returns a sequence of iterates such that for all $k$:
\[
\left\Vert x_{k}-x_{0}\right\Vert \leq2^{k}d_{0}.
\]
\end{defn}
\begin{thm}
D-Adaptation, DoG and Prodigy are exponentially bounded.
\end{thm}

A simple lower complexity bound can be established  via a simple 1-dimensional resisting oracle. The bound depends on the "scale" of the initial step of the algorithm, which is the size of the initial step from $x_0$ to $x_1$. This initial step is $g_0 \cdot d_0/\sqrt{G^2 + \left\Vert g_{0}\right\Vert ^{2}}$ for D-Adaptation, and can be thought of as an algorithm-agnostic measure of $d_0$.

Our lower bound allows the resisting oracle to choose a constant $D$ after seeing $x_1$, which is a much stronger oracle than required for establishing a lower bound. Ideally, a lower bound could be established where the constant $D$ is fixed but unknown to the algorithm, and the actual distance to solution $\left\Vert x_{0}-x_{*}\right\Vert \leq D$ given by the oracle is allowed to depend on the iterate sequence. 

The primary consequence of this difference is that our construction only tells us that hard problems exist for $n$ small relative to $D/d_0$, of the scale $n < \log(D/d_0)$. It remains an open problem to show a more general lower bound for larger $n$. This is in a sense a trivial consequence of the exponentially bounded property, but is actually representative of the real behavior of the methods during the early steps of the algorithm, where both $n$ and $d_k$ actually are small. Any more general lower bound must cover this case.

Our new D-Adaptation variants are optimal among exponentially bounded algorithms for this complexity class:
\begin{thm}
\label{thm:explb}Consider any exponentially bounded algorithm for minimizing a convex $G$-Lipschitz function
starting from $x_{0}$, which has no knowledge of problem constants G and D. There exists a fixed gradient oracle such that any sequence
of $x_{1,\dots},x_{n}$, there exists a convex Lipschitz problem
$f$ with $G=1$ and $\left\Vert x_{0}-x_{*}\right\Vert \leq D$ for all minimizing points $x_*$, consistent with the gradient oracle such that:
\[
\min_{k\leq n}f(x_{k})-f_{*}\geq\frac{DG\sqrt{\log_{2}(D/x_{1})}}{2\sqrt{n+1}}.
\]
\end{thm}

Using the simple construction from Theorem~\ref{thm:explb}, 
we show in Appendix~\ref{sec:lowertheory} that the
class of exponentially bounded methods (potentially with an exponent other than 2) covers all Gradient Descent approaches that
use an estimate of $d_{k}\le cD$ for some constant c, and use a step
size $\gamma_{k}\leq d_{k}/G$ without line-search or other additional queries. So the only way to achieve a $\log\log$
dependence on $d_{0}$ is by using a method that performs some queries that overshoot the
standard $D/G$ step size by more than a fixed constant factor. Although using larger step sizes is not problematic
for Lipschitz functions, it comes with the risk of causing training
divergence when applied to functions whose gradients are only locally
bounded by $G$, which is common in machine learning settings.

Lower complexity bounds for the average regret in the more general
online learning setting also apply here. They are of the form \citep{pdecoin}:
\[
\frac{1}{n}\sum_{k=0}^{n}\left\langle g_{k},x_{k}-x_{*}\right\rangle = \Omega\left(\frac{DG\sqrt{\log_{2}(D/\epsilon)}+\epsilon}{\sqrt{n+1}}\right).
\]
where $\epsilon$ is a ``user-specificed constant'' playing a similar role to $x_1$. Bounds on the average regret directly bound function value sub-optimality as 
\[
f(\bar{x})-f_{*}\leq\frac{1}{n+1}\sum_{k=0}^{n}\left[f(x_{k})-f_{*}\right]\leq\frac{1}{n+1}\sum_{k=0}^{n}\left\langle g_{k},x_{k}-x_{*}\right\rangle,
\]
where $\bar{x}=\frac{1}{n+1}\sum_{k=0}^{n}x_{k}$.

\section{Related Work}
In this section, we review the major classes of techniques for optimizing convex Lipschitz functions with some level of problem parameter independence. 

The Polyak step size \cite{polyakbook} trades the knowledge of $D$ for $f_{*}$, achieving optimal convergence rate without additional log factors. Stable convergence requires accurate $f_*$ estimates. A restarting scheme converges within a multiplicative log factor of the optimal rate \cite{revisiting-polyak}. There has been substantial recent research on modifications of the Polyak step size to make it better suited to machine learning tasks \citep{loizou2021stochastic, gower2021stochastic, orvieto2022dynamics} but as of yet they have not seen widespread adoption.

Coin-betting \cite{coin-betting, pmlr-v35-mcmahan14, pmlr-v75-cutkosky18a, pdecoin, varcoh} is a family of approaches from the online learning setting which are also applicable for convex non-smooth optimization. They work by establishing a relationship by duality between regret minimization and wealth-maximization. Existing approaches for wealth-maximization can then be mapped to algorithms for regret minimization. Coin-betting approaches give convergence rates for an equal-weighted average of the iterates of the form:
\[
f(\bar{x}_{n})-f_{*}=\mathcal{O}\left(\frac{DG\sqrt{\log\left(1+D/d_{0} \right)}}{\sqrt{n+1}}\right).
\]
Standard D-Adaptation obtains asymptotic rates without the log factor, but was otherwise (theoretically) slower in finite time, as it had a $\log(\cdot)$ rather than a $\sqrt{\log(\cdot)}$ dependence on $D/d_0$:
\[
f(\hat{x}_{n})-f_{*} \le \frac{16\log_{2+}(d_{n+1}/d_{0})}{n+1}D\sqrt{\sum_{k=0}^{n}\left\Vert g_{k}\right\Vert ^{2}}\leq
\frac{16DG\log_{2+}(D/d_{0})}{\sqrt{n+1}}.
\]
The Prodigy method closes the gap, giving the same sqrt-log dependence as coin betting.

The DoG method \citep{dog}, based on the bisection approach of \citet{parameterfreesgd}, is the only other approach that we are aware of that estimates $D$ in an online fashion. DoG estimates $D$ by $\bar{r}_{k}$:
\[
\bar{r}_{k}=\max_{i\leq k}\left\Vert x_{i}-x_{0}\right\Vert.
\]
\citet{dog} use this quantity as a plug-in estimate for the numerator of the step size, similar to D-Adaptation's approach. This approach can divergence in theory, but with additional modifications to the step size, the "tamed" T-DoG method is shown to converge. It has a $\log_+(D/d_0)$ dependence on the initial sub-optimally of the D estimate, so our approach improves on this dependence by a $\sqrt{\log_+(D/d_o)}$ factor.
 
\citet{malitsky20adaptive} proposed AdGD, a method for convex optimization that does not require any hyperparameters and has a rate that is at least as good as that of the optimally tuned Gradient Descent. However, AdGD requires the objective to be locally smooth, which hinders its use in many practical problems. \citet{latafat2023adaptive} partially addressed this gap by proposing a proximal extension, but the case of non-smooth differentiable functions has remained unstudied.

\section{Deriving Adam-Like Step Sizes}\label{sec:adam}
\begin{algorithm}[t]
\begin{algorithmic}[1]
    \State {\bfseries Input:} $d_0>0$ (default $10^{-6}$), $x_0$,  $\beta_1$  (default $0.9$), $\beta_2$  (default $0.999$), $\epsilon$  (default $10^{-8}$), $\gamma_k$ (default 1 with cosine annealing)
    \State $r_0=0$, $s_0=0$, $m_0=0$, $v_0=0$
    \For{$k=0$ {\bfseries to} $n$}
    \State $g_{k} \in \partial f(x_{k})$
        \State  $m_{k+1}= \beta_1 m_k + (1-\beta_1) d_k g_k$
        \State  $v_{k+1}= \beta_2 v_k + (1-\beta_2) d_k^2g_k^2$
        \State $r_{k+1} = \sqrt{\beta_2}r_k + (1-\sqrt{\beta_2})\gamma_k d_k^2\langle g_k, x_0 - x_k\rangle$
        \State $s_{k+1} = \sqrt{\beta_2}s_k + (1-\sqrt{\beta_2})\gamma_k d_k^2 g_k$
	\State $\hat{d}_{k+1} = \dfrac{r_{k+1}}{\left\Vert s_{k+1}\right\Vert_1 }$
        \State $d_{k+1}= \max(d_k, \hat d_{k+1})$
        \State $x_{k+1} = x_k - \gamma_k d_{k} m_{k+1}/(\sqrt{v_{k+1}}+d_k\epsilon)$
    \EndFor
\end{algorithmic}
\caption{\label{alg:prodigy_adam}Prodigy (Adam version)}
\end{algorithm}
To derive an Adam-like method, which should use an exponential moving average for the step size, we want to approximate Adam's update of the exponential moving average of squared gradients:
\[
    v_{k+1} = \beta_2 v_{k} + (1-\beta_2)g_k^2
    = (1-\beta_2) \sum_{i=0}^k \beta_2^{k-i} g_i^2,
\]
where $g_k^2$ is the coordinate-wise square of the gradient $g_k$. We can achieve this using exponential weights, $\lambda_k = \beta_2^{-k/2}$, which after substituting the definition of $\eta_k$ give us the following identity:
\[
    \frac{d_k^4}{\eta_k^2}
    = \frac{d_k^2}{\lambda_k^2}G^2 + \sum_{i=0}^{k}d_i^2\frac{\lambda_i^2}{\lambda_k^2} \|g_i\|^2
    = \frac{d_k^2}{\lambda_k^2}G^2 + d_k^2\|g_k\|^2 + \sum_{i=0}^{k-1}\beta_2^{k-i}d_i^2\|g_i\|^2.
\]
This can be seen as computing an exponential moving average of $d_k g_k$ rather than $g_k$ itself. This is our first observation. In addition, in Appendix~\ref{sec:coordinate}, we provide a coordinate-wise version of Algorithm~\ref{alg:dadagradv2da} and study its convergence properties. Based on the theory presented there, the denominator for $\hat d_{k+1}$ should use the $\ell_1$ norm of the weighted gradient sum. Thus, combining this insight with the design of Algorithm~\ref{alg:dadagradv2gd}, we obtain the following expression for the Adam estimate of $D$:
\[
    \hat{d}_{k+1}
    = \frac{\sum_{i=0}^k \lambda_id_i^2\langle g_i, x_0 - x_i\rangle}{\Vert \sum_{i=0}^k \lambda_id_i^2 g_i\Vert_1}
    = \frac{\sum_{i=0}^k\beta_2^{(k-i)/2}d_i^2 \langle g_i, x_0 - x_i\rangle}{\Vert \sum_{i=0}^k\beta_2^{(k-i)/2}d_i^2 g_i\Vert_1}.
\]
The update uses exponential moving average as well, although it is more conservative as it uses $\sqrt{\beta_2}$ instead of $\beta_2$. Note that there is an extra of $(1-\beta_2)$ in the update for $v_k$, which can be optionally compensated for by using the bias correction discussed by \citet{kingma2015adam}. These update rules are summarized in Algorithm~\ref{alg:prodigy_adam}. This is the main algorithm that we study numerically in the next section.

\section{Experiments}
We test our methods on convex logistic regression as well as deep learning problems.
The Prodigy method is used as presented in Algorithm~\ref{alg:prodigy_adam} in all deep learning experiments.

\paragraph{Logistic regression.} For the convex setting, we ran a set of classification experiments. For each dataset, we used the multi-margin loss \citep{multimargin}, a multi-class generalization of the hinge loss. This non-smooth loss results in bounded gradients, which is required by our theory. We perform full-batch rather that stochastic optimization, for two reasons. Firstly, it matches the assumptions of our theory. Secondly, fast learning rate adaptation is more crucial for full-batch optimization than stochastic optimization as fewer total steps will be performed. 

We performed 1,000 steps for each dataset, using a randomized $x_{0}$ and plot the results of 10 seeds. We ran both DA and SGD variants of each method. Each plot shows the accuracy of the average iterate for each method. Figure~\ref{fig:convex} shows that our proposed algorithm greatly out-performs regular D-Adaptation. Our weighted SGD variant of D-Adaptation is faster consistently across each dataset. Additionally, it adapts faster than the DoG method \citep{dog} on 10 of the 12 problems.

\paragraph{CIFAR10.} For neural network experiments\footnote{The PyTorch code of our optimizer is available at \url{https://github.com/konstmish/prodigy}}, we consider training on CIFAR10 \citep{cifar} with batch size 256, where D-Adapted Adam has a gap of a few percent compared to the standard Adam. We use cosine annealing with initial step size 1 for all Adam-based methods and initial step size $10^{-3}$ for Adam itself. The considered networks are VGG11 \citep{simonyan2014very} and ResNet-50~\citep{he2016deep}\footnote{VGG11 and ResNet-50 implementation along with the data loaders were taken from \url{https://github.com/kuangliu/pytorch-cifar}}. To simplify the experiment, we do not use weight decay, so both networks slightly overfit and do not reach high test accuracy values. All methods were run using same 8 random seeds.

We show the results in Figure~\ref{fig:cifar10}. As we can see, this gap is closed by Prodigy, which is achieved by the larger estimates of the step size.
 
For DoG and L-DoG, we compute the polynomial-averaging iterate and then report the best of the average and the last iterate. We average with $\gamma=8$, see \citep{dog} for the details. While DoG produces larger step size estimate than Prodigy (see the right column in Figure~\ref{fig:cifar10}, this is counterbalanced by the larger denominator in DoG. We also tried to modify DoG to use Adam-like step sizes but our heuristic modification diverged on this problem. We also observed that among DoG and its layer-wise version, L-DoG, there is no clear winner as the former performed better on VGG11 and the latter was better when training ResNet-50.

\begin{figure}
\includegraphics[width=\textwidth]{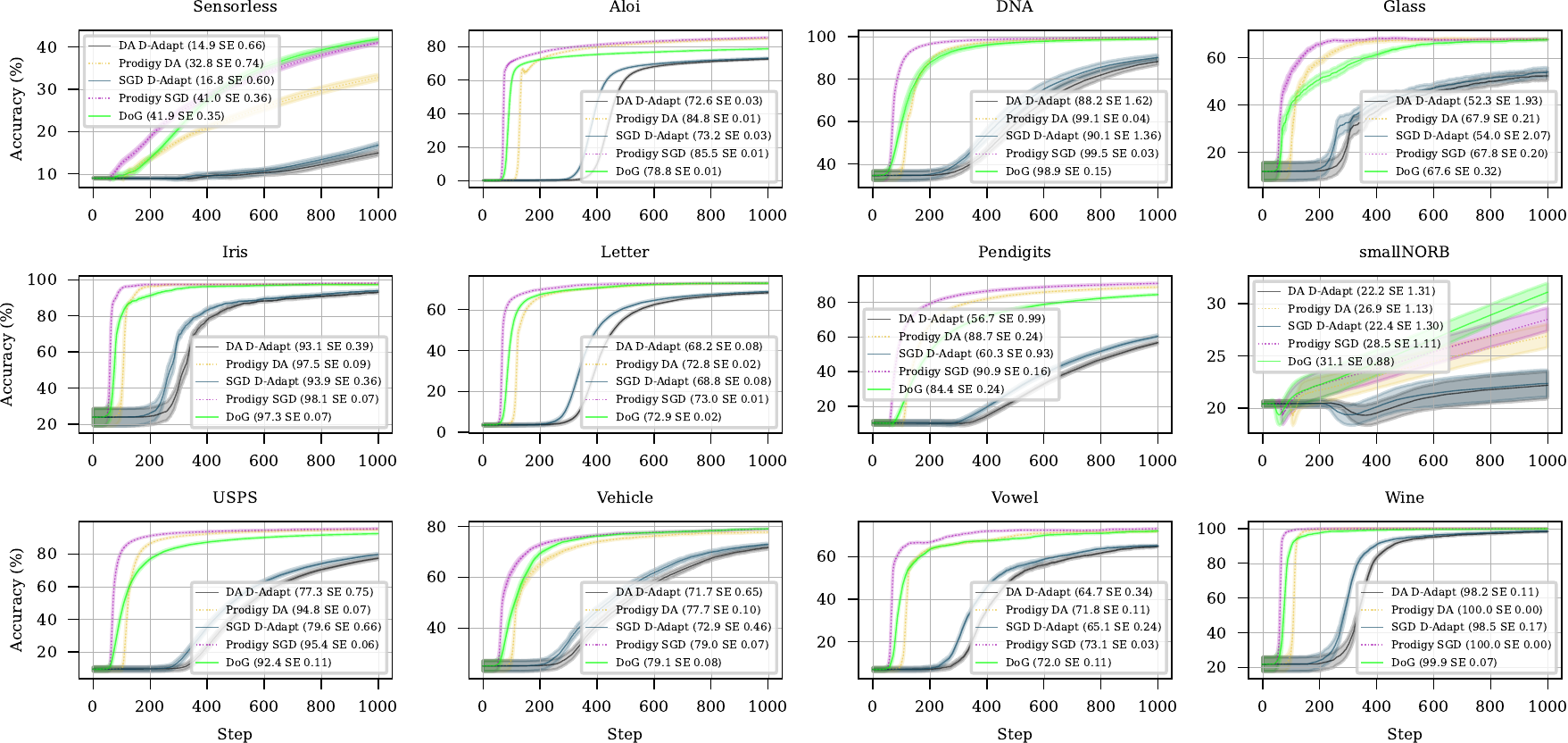}
\caption{\label{fig:convex} Convex multiclass classification problems. Error bars show a range of 1 standard error above and below the mean of the 10 seeds.}
\end{figure}

\begin{figure}
    \centering
    \includegraphics[scale=0.3]{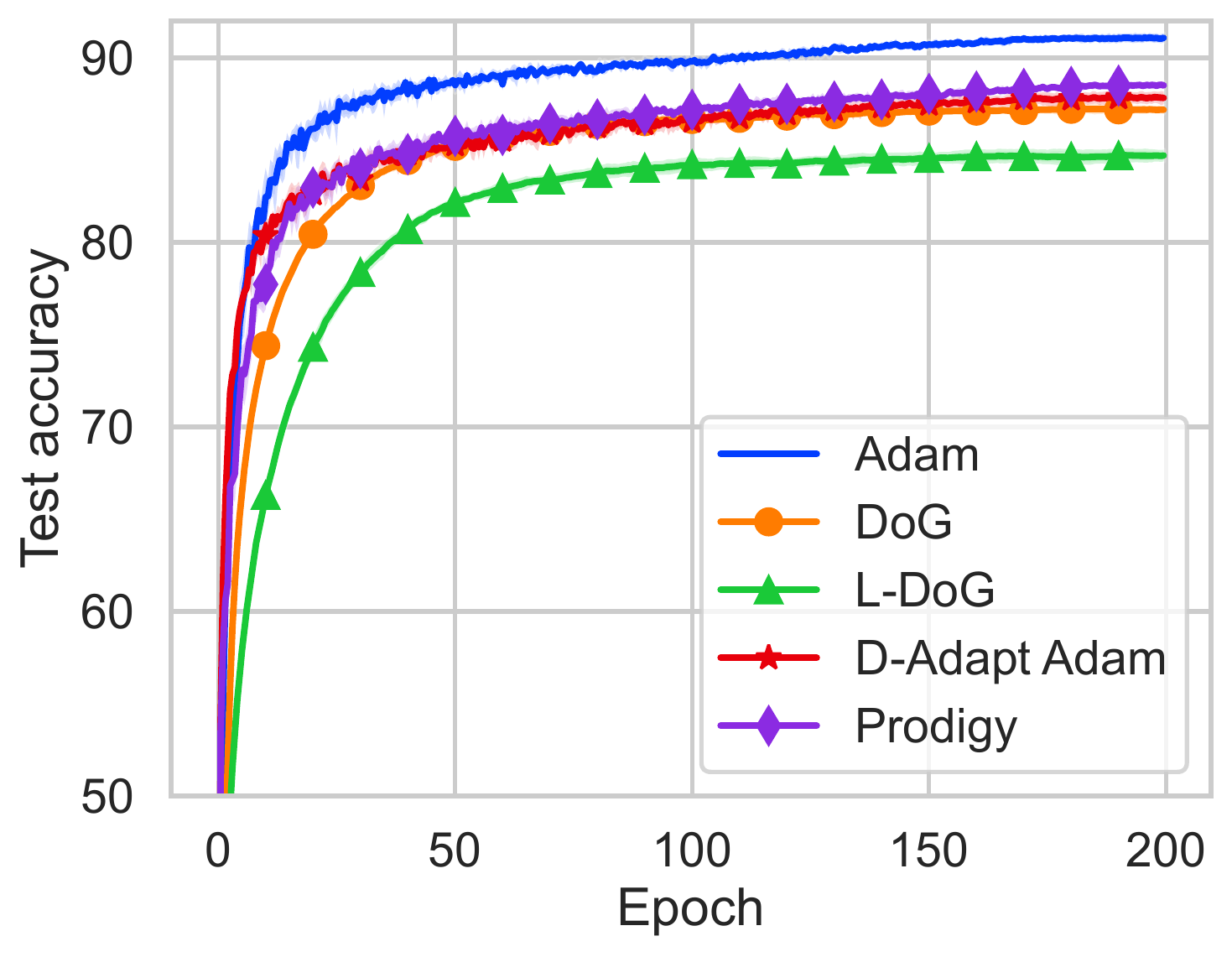}
    \includegraphics[scale=0.3]{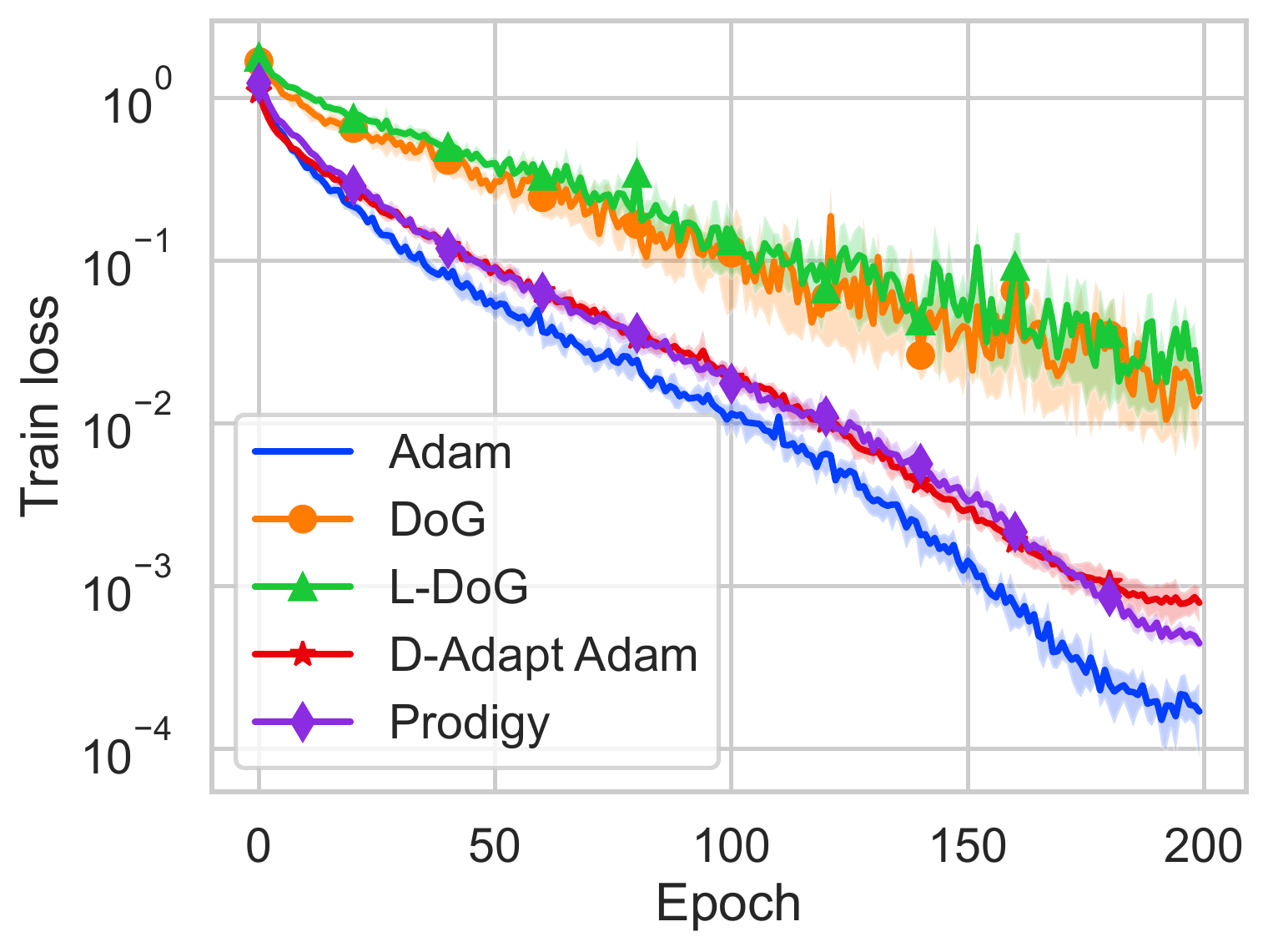}
    \includegraphics[scale=0.3]{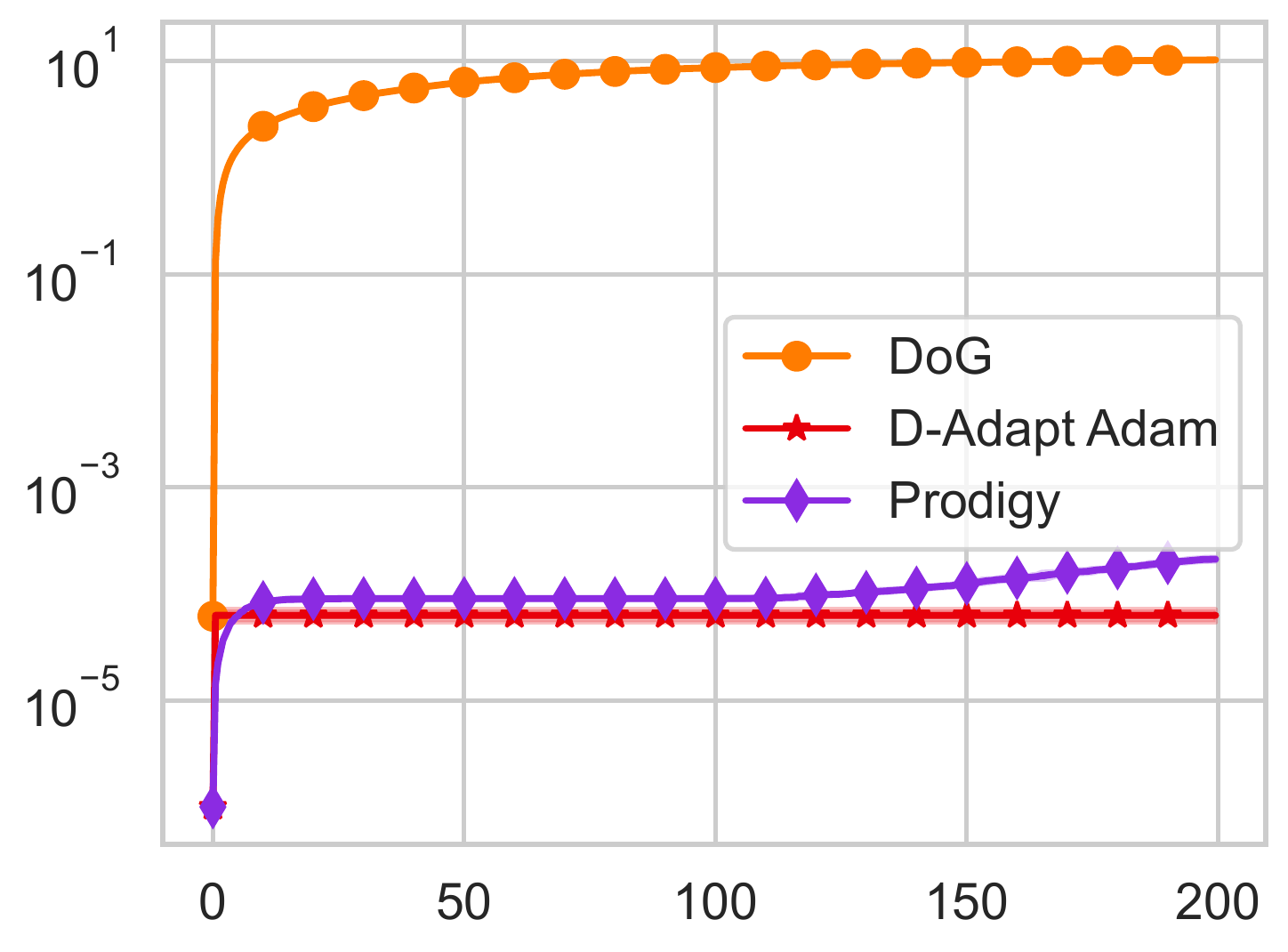}\\
    \includegraphics[scale=0.3]{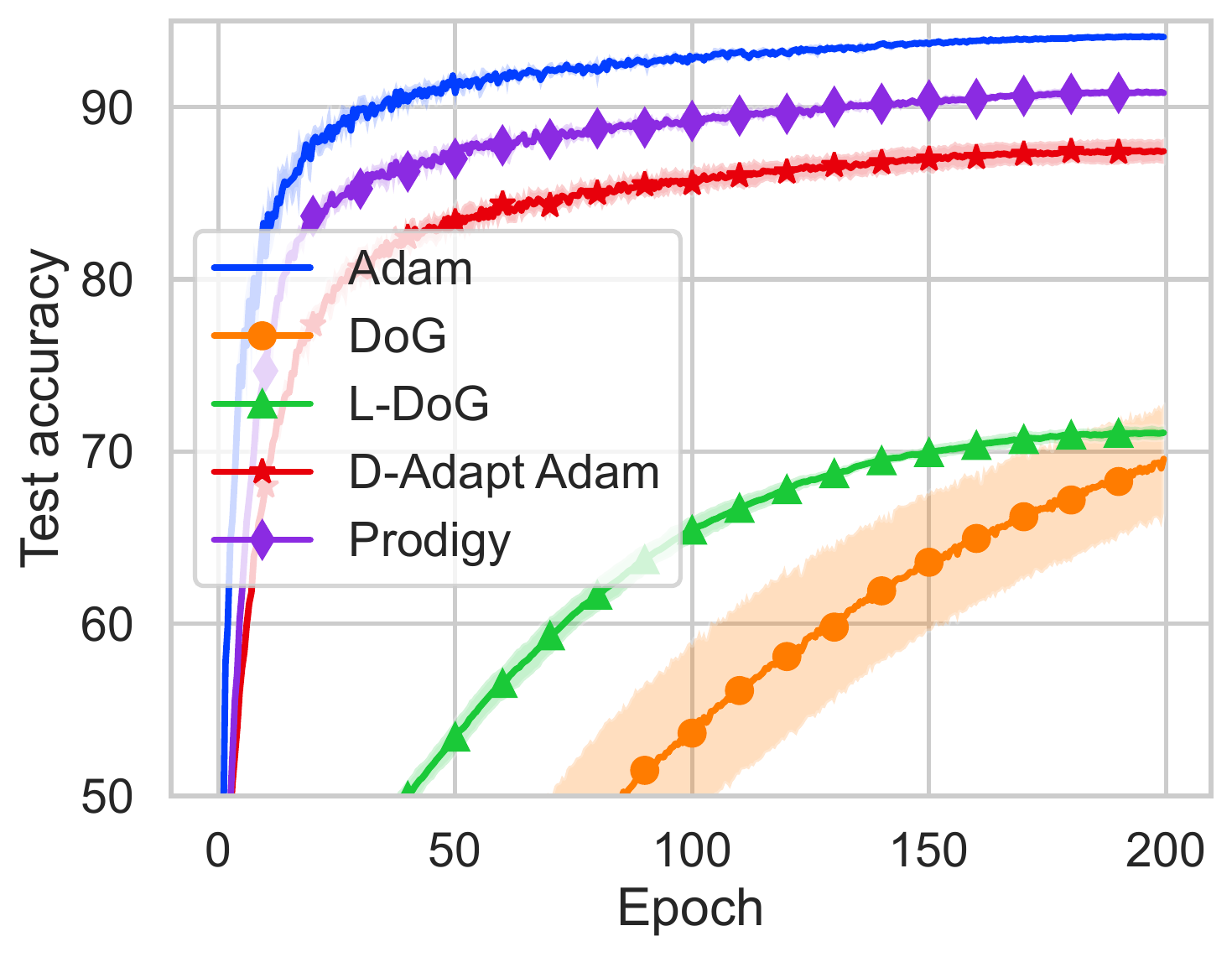}
    \includegraphics[scale=0.3]{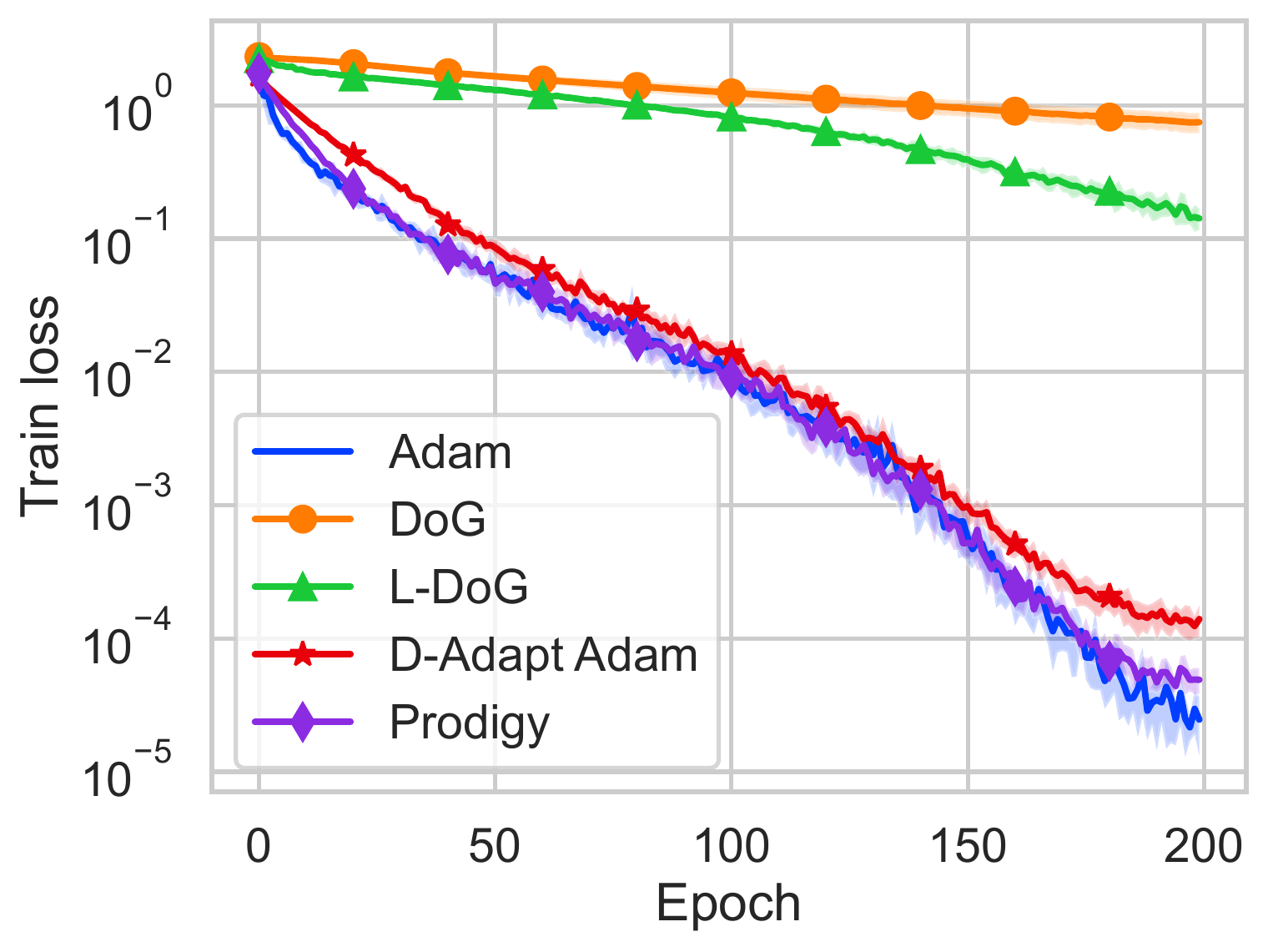}
    \includegraphics[scale=0.3]{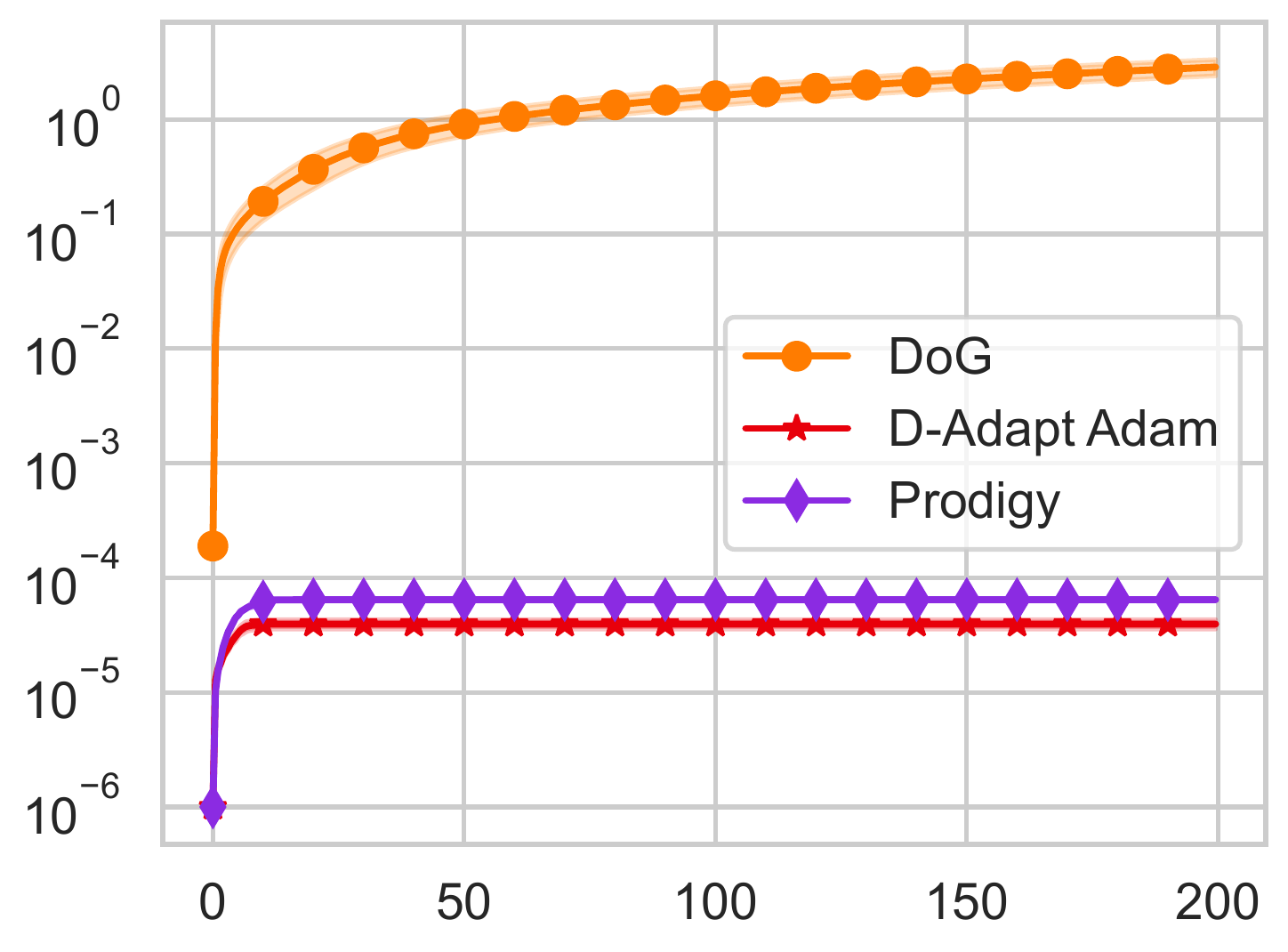}
    \caption{VGG11 and ResNet-50 training on CIFAR10. Left: test accuracy (\%), middle: train loss, right: step sizes. ``Prodigy'' is used as given in Algorithm~\ref{alg:prodigy_adam}. As expected, Prodigy estimates a larger step size than D-Adaptation, which helps it reach test accuracy closer to the one of Adam.}
    \label{fig:cifar10}
\end{figure}

\paragraph{nanoGPT transformer.} \begin{figure}[t]
    \centering
    \includegraphics[scale=0.29]{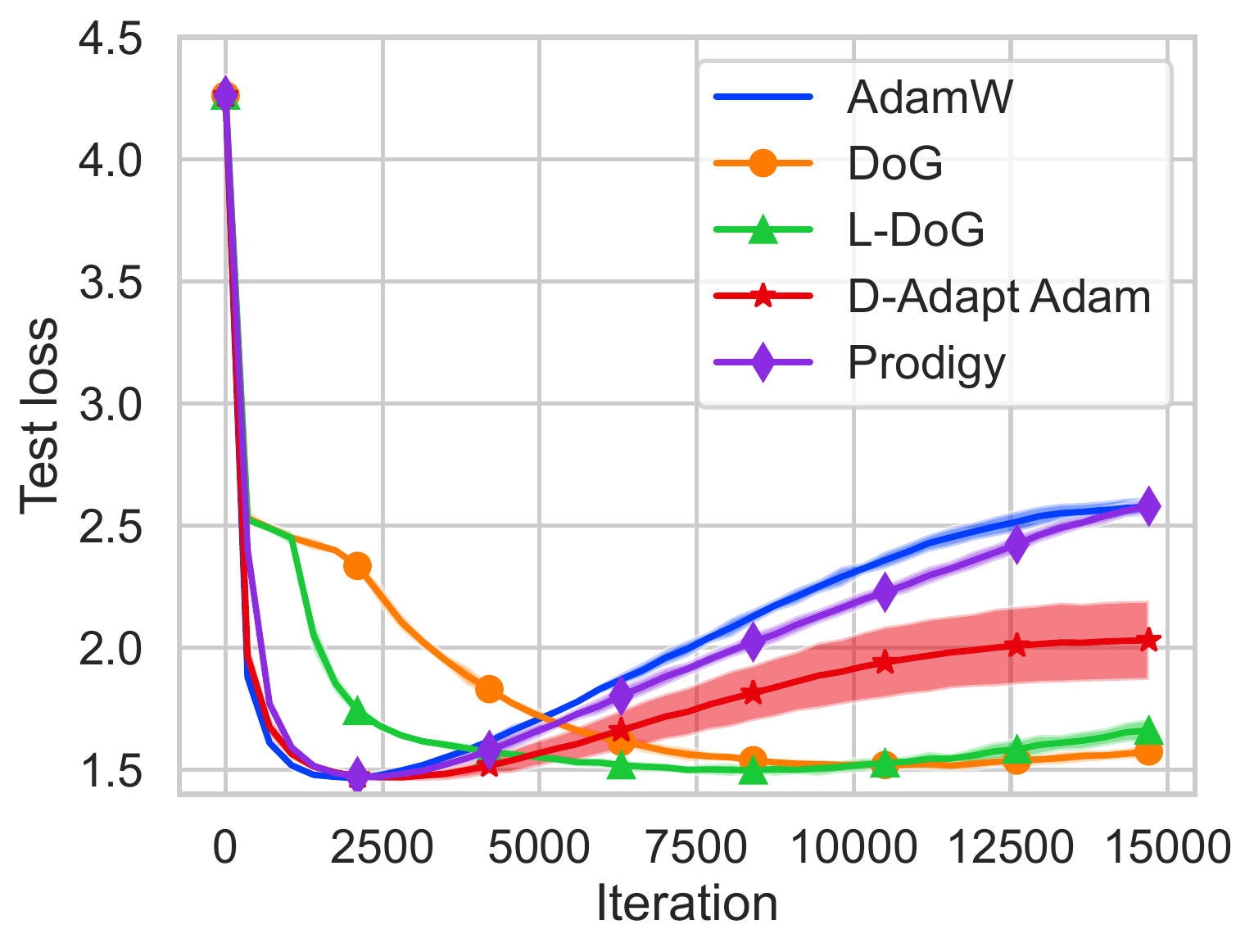}
    \includegraphics[scale=0.29]{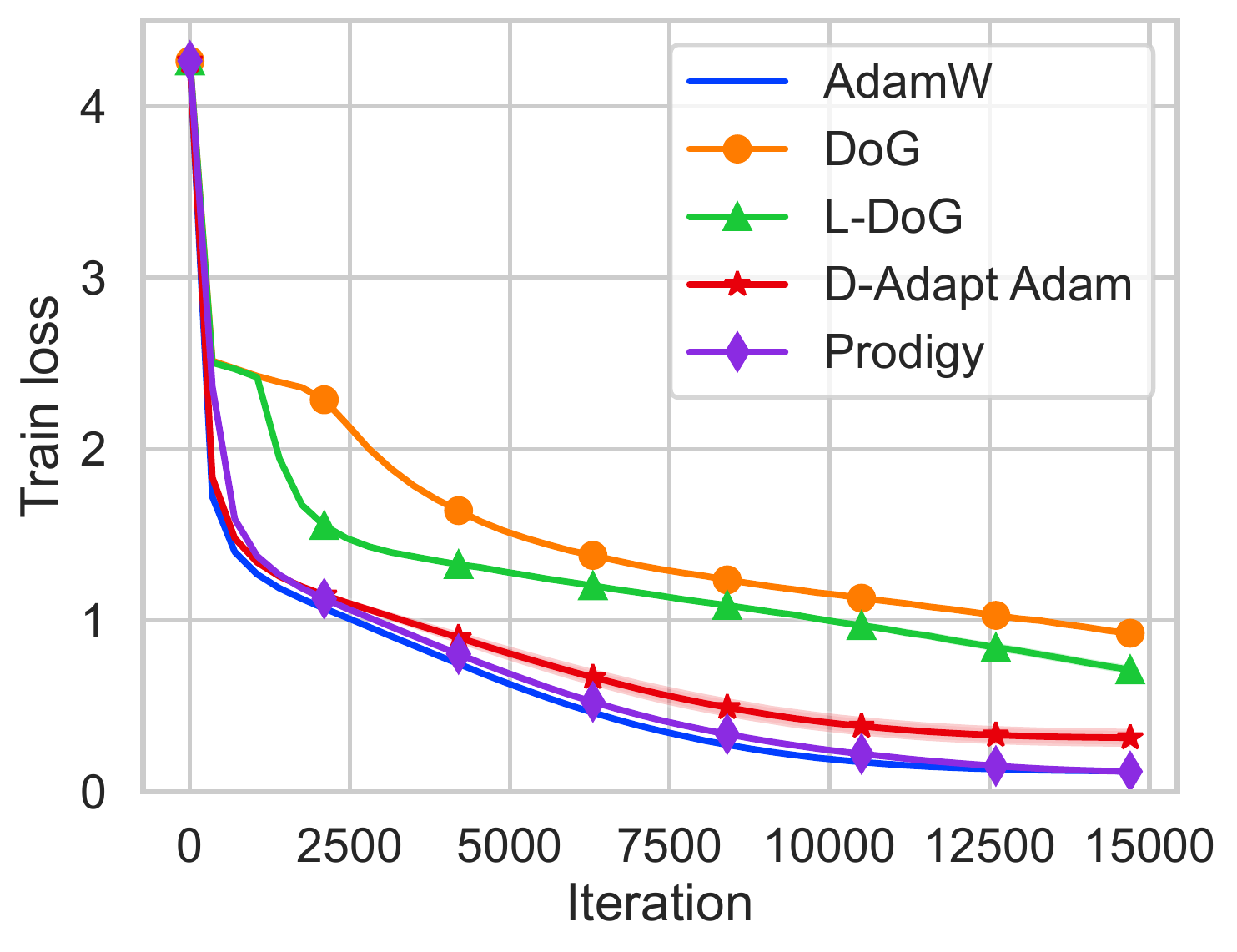}
    \includegraphics[scale=0.29]{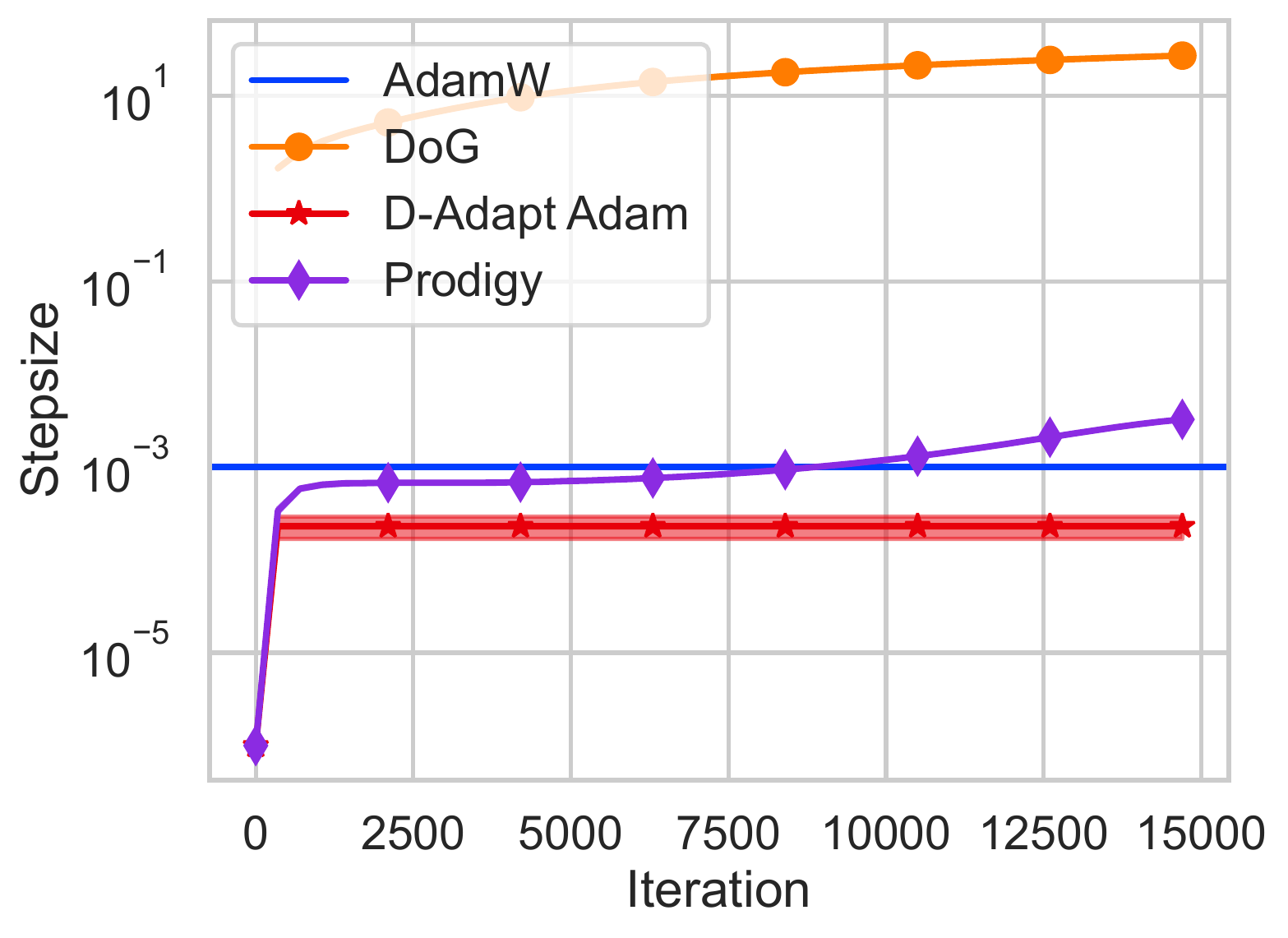}
    \caption{The test (left) and train (middle) loss curves  as well as the estimated stepsize (right) when training a 6-layer nanoGPT transformer on the Shakespeare dataset.}
    \label{fig:shakespeare}
\end{figure}
We also train a 6-layer transformer network from nanoGPT\footnote{\url{https://github.com/karpathy/nanoGPT}} on the Shakespeare dataset. For all methods, we use batch size 256, clip the gradients to have norm not exceeding 1 and use float16 numbers. We use AdamW with hyperparameters given in the repository, i.e., $\beta_2=0.99$, weight decay $0.1$, stepsize $10^{-3}$, cosine annealing with warmup over 100 steps. The same weight decay value and cosine annealing is used for Prodigy and D-Adapted Adam, except that the latter two methods use stepsize 1. We accumulate minibatches of size 12 into a batch of size 480. We tuned the weight decay for DoG and L-DoG and found the value $10^{-4}$ to work well for this problem. We ran each method with 8 random seeds and report the average as well as one-standard-deviation confidence intervals.

See Figure~\ref{fig:shakespeare} for the results. In terms of the test loss, all methods are roughly equivalent except that DoG and L-DoG were slower to reach the best value of roughly 1.5. For the train loss, Prodigy was on par with tuned AdamW and slightly better than D-Adapted Adam. Surprisingly, the estimated step size in Prodigy was very consistent across the 8 random seeds.

\subsection{Large-scale Adam experiments}
To validate the performance on large-scale practical applications directly against D-Adaptation, we ran the subset of the experiments from \citet{defazaio2023learning} that use the Adam optimizer. Methods without coordinate adaptivity are not competitive on these problems and so we exclude SGD and DoG from these comparisons.

\paragraph{LSTM, RoBERTa, GPT, DLRM, VarNet.} On the smallest problem of LSTM training, Prodigy appears to converge significantly faster in training loss and slightly overfits in test loss compared to the baselines. For RoBERTa~\citep{liu2019roberta} and GPT~\citep{gpt} training on BookWiki, Prodigy matches the performance of the baseline with only negligible differences. For the application problems, DLRM~\citep{DLRM19} on the Criteo Kaggle Display Advertising dataset, and fastMRI VarNet~\citep{zbontar2018fastmri}, Prodigy again closely matches the baselines.

\paragraph{ViT training.} \citet{defazaio2023learning} present a negative result for training vision transformer~\citep{dosovitskiy2020image}, where D-Adaptation significantly underperforms tuned Adam. We investigated this effect, and we were able to reproduce this gap across a wide range of weight-decay values, although this problem has high run-to-run variance of 1-2\% of test accuracy, which makes comparison difficult. Using weight decay 0.05 instead of 0.1 significantly improved performance of each variant, and so we present results for both the baselines and Prodigy at that value. We can see that Prodigy almost closes the gap between tuned Adam and D-Adaptation, giving a test accuracy of 74.63\% compared to 75.4\% for Adam, and more than 2\% higher than D-Adaptation. See Figure~\ref{fig:dlrm_fastmri_vit} for the results.

\begin{figure}
\includegraphics[width=0.49\textwidth]{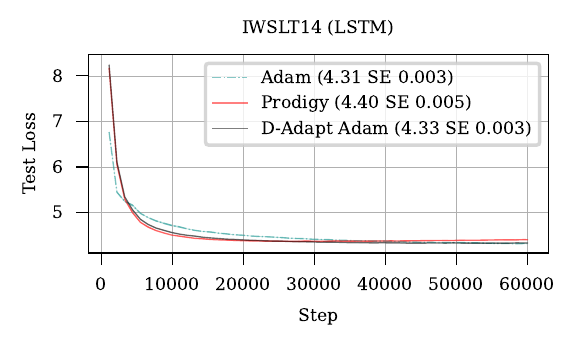}\includegraphics[width=0.49\textwidth]{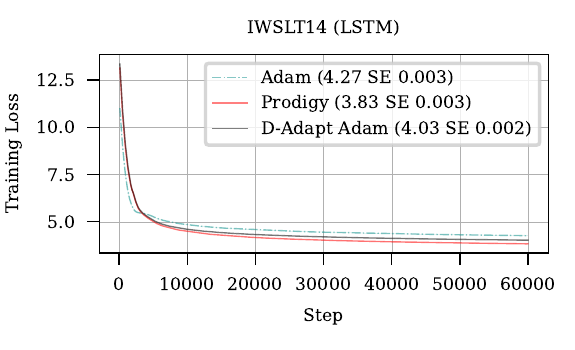}
\includegraphics[width=0.49\textwidth]{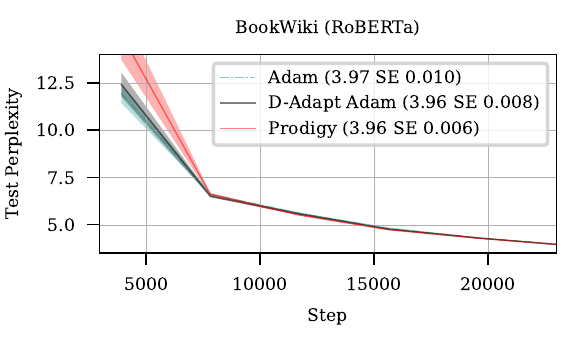}\includegraphics[width=0.49\textwidth]{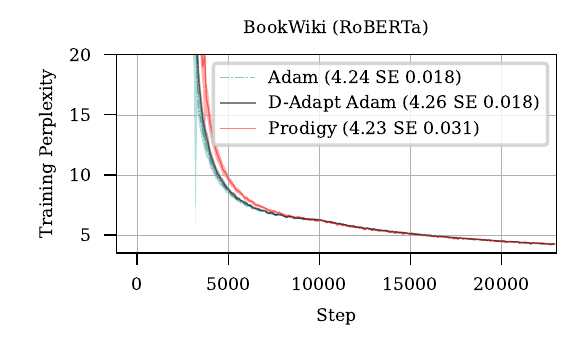}
\includegraphics[width=0.49\textwidth]{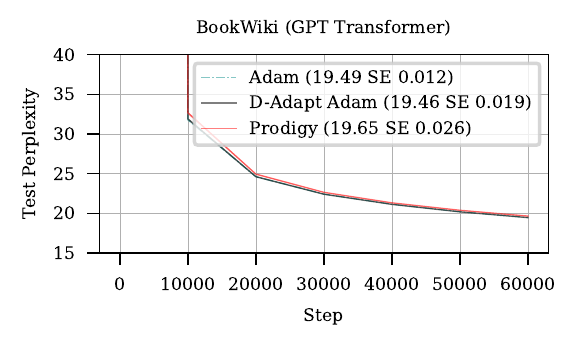}\includegraphics[width=0.49\textwidth]{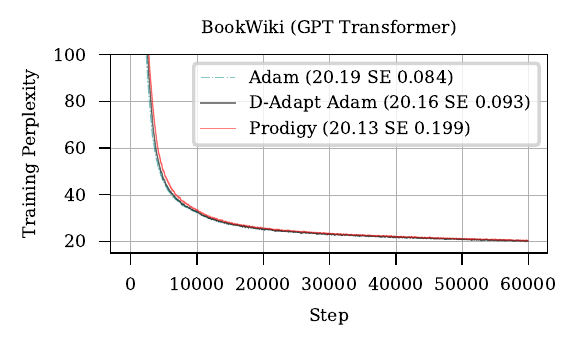}
\caption{\label{fig:lstm}Adam-family experiments.}
\end{figure}

\begin{figure}\label{fig:dlrm_fastmri_vit}
\includegraphics[width=0.49\textwidth]{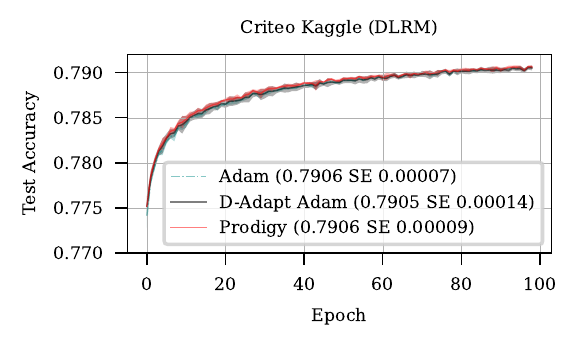}\includegraphics[width=0.49\textwidth]{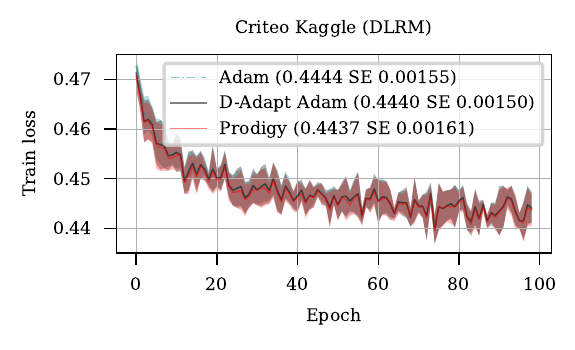}
\includegraphics[width=0.49\textwidth]{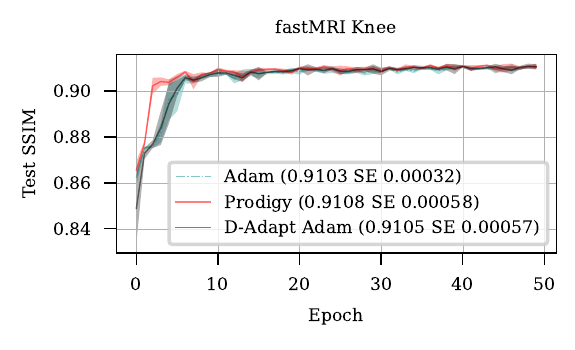}\includegraphics[width=0.49\textwidth]{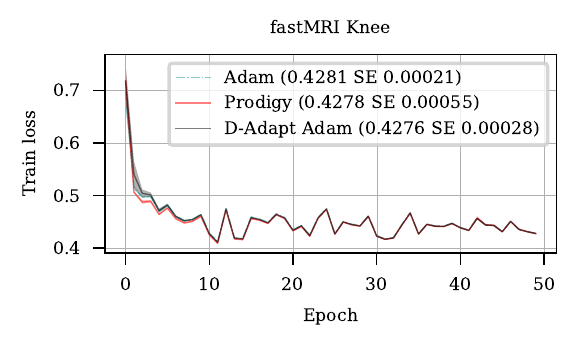}
\includegraphics[width=0.49\textwidth]{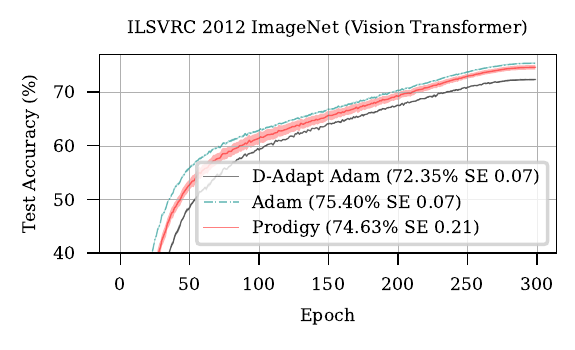}\includegraphics[width=0.49\textwidth]{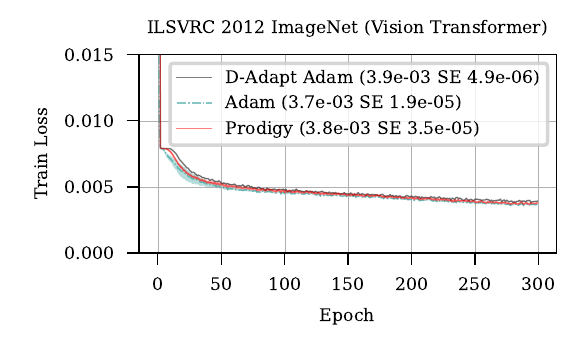}
\caption{\label{fig:other}Adam-family experiments.}
\end{figure}

\section{Conclusion}
We have presented two new methods for learning rate adaptation that improve upon the adaptation rate of the state-of-the-art D-Adaptation method. Prodigy, a form of weighted D-Adaptation, was shown to adapt faster than other known methods across a range of experiments.

\clearpage

\bibliography{expeditiously}

\begin{thebibliography}{35}
\providecommand{\natexlab}[1]{#1}
\providecommand{\url}[1]{\texttt{#1}}
\expandafter\ifx\csname urlstyle\endcsname\relax
  \providecommand{\doi}[1]{doi: #1}\else
  \providecommand{\doi}{doi: \begingroup \urlstyle{rm}\Url}\fi

\bibitem[Carmon and Hinder(2022)]{parameterfreesgd}
Yair Carmon and Oliver Hinder.
\newblock Making {SGD} parameter-free.
\newblock In \emph{Conference on Learning Theory}. PMLR, 2022.

\bibitem[Cutkosky and Orabona(2018)]{pmlr-v75-cutkosky18a}
Ashok Cutkosky and Francesco Orabona.
\newblock Black-box reductions for parameter-free online learning in banach
  spaces.
\newblock In \emph{Proceedings of the 31st Conference On Learning Theory},
  Proceedings of Machine Learning Research. PMLR, 2018.

\bibitem[Defazio and Jelassi(2022)]{defazio2022adaptivity}
Aaron Defazio and Samy Jelassi.
\newblock Adaptivity without compromise: A momentumized, adaptive, dual
  averaged gradient method for stochastic optimization.
\newblock \emph{Journal of Machine Learning Research}, 23:\penalty0 1--34,
  2022.

\bibitem[Defazio and Mishchenko(2023)]{defazaio2023learning}
Aaron Defazio and Konstantin Mishchenko.
\newblock Learning-rate-free learning by {D}-adaptation.
\newblock In Andreas Krause, Emma Brunskill, Kyunghyun Cho, Barbara Engelhardt,
  Sivan Sabato, and Jonathan Scarlett, editors, \emph{Proceedings of the 40th
  International Conference on Machine Learning}, volume 202 of
  \emph{Proceedings of Machine Learning Research}, pages 7449--7479. PMLR,
  23--29 Jul 2023.

\bibitem[Dosovitskiy et~al.(2021)Dosovitskiy, Beyer, Kolesnikov, Weissenborn,
  Zhai, Unterthiner, Dehghani, Minderer, Heigold, Gelly, Uszkoreit, and
  Houlsby]{dosovitskiy2020image}
Alexey Dosovitskiy, Lucas Beyer, Alexander Kolesnikov, Dirk Weissenborn,
  Xiaohua Zhai, Thomas Unterthiner, Mostafa Dehghani, Matthias Minderer, Georg
  Heigold, Sylvain Gelly, Jakob Uszkoreit, and Neil Houlsby.
\newblock An image is worth 16x16 words: Transformers for image recognition at
  scale.
\newblock In \emph{International Conference on Learning Representations}, 2021.

\bibitem[Duchi et~al.(2011)Duchi, Hazan, and Singer]{adagrad}
John Duchi, Elad Hazan, and Yoram Singer.
\newblock Adaptive subgradient methods for online learning and stochastic
  optimization.
\newblock \emph{Journal of Machine Learning Research}, 12\penalty0 (61), 2011.

\bibitem[Goodfellow et~al.(2020)Goodfellow, Pouget-Abadie, Mirza, Xu,
  Warde-Farley, Ozair, Courville, and Bengio]{goodfellow2020generative}
Ian Goodfellow, Jean Pouget-Abadie, Mehdi Mirza, Bing Xu, David Warde-Farley,
  Sherjil Ozair, Aaron Courville, and Yoshua Bengio.
\newblock Generative adversarial networks.
\newblock \emph{Communications of the ACM}, 63\penalty0 (11):\penalty0
  139--144, 2020.

\bibitem[Gower et~al.(2021)Gower, Defazio, and Rabbat]{gower2021stochastic}
Robert~M. Gower, Aaron Defazio, and Michael Rabbat.
\newblock Stochastic {Polyak} stepsize with a moving target.
\newblock \emph{arXiv preprint arXiv:2106.11851}, 2021.

\bibitem[Hazan and Kakade(2019)]{revisiting-polyak}
Elad Hazan and Sham~M. Kakade.
\newblock Revisiting the {Polyak} step size.
\newblock \emph{arXiv preprint arXiv:1905.00313}, 2019.

\bibitem[He et~al.(2016)He, Zhang, Ren, and Sun]{he2016deep}
Kaiming He, Xiangyu Zhang, Shaoqing Ren, and Jian Sun.
\newblock Deep residual learning for image recognition.
\newblock In \emph{Proceedings of the IEEE conference on computer vision and
  pattern recognition}, 2016.

\bibitem[Ivgi et~al.(2023)Ivgi, Hinder, and Carmon]{dog}
Maor Ivgi, Oliver Hinder, and Yair Carmon.
\newblock {D}o{G} is {SGD}’s best friend: A parameter-free dynamic step size
  schedule.
\newblock In Andreas Krause, Emma Brunskill, Kyunghyun Cho, Barbara Engelhardt,
  Sivan Sabato, and Jonathan Scarlett, editors, \emph{Proceedings of the 40th
  International Conference on Machine Learning}, volume 202 of
  \emph{Proceedings of Machine Learning Research}, pages 14465--14499. PMLR,
  23--29 Jul 2023.

\bibitem[Kairouz et~al.(2021)Kairouz, McMahan, Avent, Bellet, Bennis, Bhagoji,
  Bonawitz, Charles, Cormode, Cummings, et~al.]{kairouz2019advances}
Peter Kairouz, H.~Brendan McMahan, Brendan Avent, Aur{\'e}lien Bellet, Mehdi
  Bennis, Arjun~Nitin Bhagoji, Keith Bonawitz, Zachary Charles, Graham Cormode,
  Rachel Cummings, et~al.
\newblock Advances and open problems in federated learning.
\newblock \emph{Foundations and Trends{\textregistered} in Machine Learning},
  14\penalty0 (1), 2021.

\bibitem[Kavis et~al.(2019)Kavis, Levy, Bach, and Cevher]{kavis2019unixgrad}
Ali Kavis, Kfir~Y. Levy, Francis Bach, and Volkan Cevher.
\newblock {UniXGrad}: A universal, adaptive algorithm with optimal guarantees
  for constrained optimization.
\newblock \emph{Advances in neural information processing systems}, 32, 2019.

\bibitem[Khodak et~al.(2021)Khodak, Tu, Li, Li, Balcan, Smith, and
  Talwalkar]{khodak2021federated}
Mikhail Khodak, Renbo Tu, Tian Li, Liam Li, Maria-Florina~F. Balcan, Virginia
  Smith, and Ameet Talwalkar.
\newblock Federated hyperparameter tuning: Challenges, baselines, and
  connections to weight-sharing.
\newblock \emph{Advances in Neural Information Processing Systems},
  34:\penalty0 19184--19197, 2021.

\bibitem[Kingma and Ba(2015)]{kingma2015adam}
Diederik~P. Kingma and Jimmy Ba.
\newblock Adam: A method for stochastic optimization.
\newblock In \emph{ICLR}, 2015.
\newblock URL \url{http://arxiv.org/abs/1412.6980}.

\bibitem[Krizhevsky(2009)]{cifar}
Alex Krizhevsky.
\newblock Learning multiple layers of features from tiny images.
\newblock Technical report, University of Toronto, 2009.

\bibitem[Latafat et~al.(2023)Latafat, Themelis, Stella, and
  Patrinos]{latafat2023adaptive}
Puya Latafat, Andreas Themelis, Lorenzo Stella, and Panagiotis Patrinos.
\newblock Adaptive proximal algorithms for convex optimization under local
  {Lipschitz} continuity of the gradient.
\newblock \emph{arXiv preprint arXiv:2301.04431}, 2023.

\bibitem[Levy et~al.(2018)Levy, Yurtsever, and Cevher]{levy2018online}
Kfir~Y. Levy, Alp Yurtsever, and Volkan Cevher.
\newblock Online adaptive methods, universality and acceleration.
\newblock \emph{Advances in neural information processing systems}, 31, 2018.

\bibitem[Liu et~al.(2019)Liu, Ott, Goyal, Du, Joshi, Chen, Levy, Lewis,
  Zettlemoyer, and Stoyanov]{liu2019roberta}
Yinhan Liu, Myle Ott, Naman Goyal, Jingfei Du, Mandar Joshi, Danqi Chen, Omer
  Levy, Mike Lewis, Luke Zettlemoyer, and Veselin Stoyanov.
\newblock {RoBERTa}: A robustly optimized {BERT} pretraining approach.
\newblock \emph{arXiv preprint arXiv:1907.11692}, 2019.

\bibitem[Loizou et~al.(2021)Loizou, Vaswani, Laradji, and
  Lacoste-Julien]{loizou2021stochastic}
Nicolas Loizou, Sharan Vaswani, Issam Laradji, and Simon Lacoste-Julien.
\newblock Stochastic {Polyak} step-size for {SGD}: An adaptive learning rate
  for fast convergence.
\newblock In \emph{Proceedings of the 24th International Conference on
  Artificial Intelligence and Statistics (AISTATS)}, Proceedings of Machine
  Learning Research. PMLR, 2021.

\bibitem[Malitsky and Mishchenko(2020)]{malitsky20adaptive}
Yura Malitsky and Konstantin Mishchenko.
\newblock Adaptive gradient descent without descent.
\newblock In Hal~Daumé III and Aarti Singh, editors, \emph{Proceedings of the
  37th International Conference on Machine Learning}, volume 119 of
  \emph{Proceedings of Machine Learning Research}, pages 6702--6712. PMLR,
  13--18 Jul 2020.

\bibitem[McMahan and Orabona(2014)]{pmlr-v35-mcmahan14}
H.~Brendan McMahan and Francesco Orabona.
\newblock Unconstrained online linear learning in {Hilbert} spaces: Minimax
  algorithms and normal approximations.
\newblock In \emph{Proceedings of The 27th Conference on Learning Theory},
  volume~35 of \emph{Proceedings of Machine Learning Research}. PMLR, 2014.

\bibitem[Naumov et~al.(2019)Naumov, Mudigere, Shi, Huang, Sundaraman, Park,
  Wang, Gupta, Wu, Azzolini, Dzhulgakov, Mallevich, Cherniavskii, Lu,
  Krishnamoorthi, Yu, Kondratenko, Pereira, Chen, Chen, Rao, Jia, Xiong, and
  Smelyanskiy]{DLRM19}
Maxim Naumov, Dheevatsa Mudigere, Hao{-}Jun~Michael Shi, Jianyu Huang,
  Narayanan Sundaraman, Jongsoo Park, Xiaodong Wang, Udit Gupta, Carole{-}Jean
  Wu, Alisson~G. Azzolini, Dmytro Dzhulgakov, Andrey Mallevich, Ilia
  Cherniavskii, Yinghai Lu, Raghuraman Krishnamoorthi, Ansha Yu, Volodymyr
  Kondratenko, Stephanie Pereira, Xianjie Chen, Wenlin Chen, Vijay Rao, Bill
  Jia, Liang Xiong, and Misha Smelyanskiy.
\newblock Deep learning recommendation model for personalization and
  recommendation systems.
\newblock \emph{CoRR}, 2019.

\bibitem[Orabona and P{\'a}l(2021)]{varcoh}
Francesco Orabona and D{\'a}vid P{\'a}l.
\newblock Parameter-free stochastic optimization of variationally coherent
  functions, 2021.

\bibitem[Orabona and Tommasi(2017)]{coin-betting}
Francesco Orabona and Tatiana Tommasi.
\newblock Training deep networks without learning rates through coin betting.
\newblock In \emph{Advances in Neural Information Processing Systems},
  volume~30, 2017.

\bibitem[Orvieto et~al.(2022)Orvieto, Lacoste-Julien, and
  Loizou]{orvieto2022dynamics}
Antonio Orvieto, Simon Lacoste-Julien, and Nicolas Loizou.
\newblock Dynamics of {SGD} with stochastic {Polyak} stepsizes: Truly adaptive
  variants and convergence to exact solution.
\newblock In \emph{Advances in Neural Information Processing Systems
  (NeurIPS)}. NeurIPS, 2022.

\bibitem[Polyak(1987)]{polyakbook}
Boris~T. Polyak.
\newblock \emph{Introduction to optimization}.
\newblock Optimization Software, Inc., 1987.

\bibitem[Radford et~al.(2019)Radford, Narasimhan, Salimans, and Sutskever]{gpt}
Alec Radford, Karthik Narasimhan, Tim Salimans, and Ilya Sutskever.
\newblock Improving language understanding by generative pre-training.
\newblock Technical report, OpenAI, 2019.

\bibitem[Simonyan and Zisserman(2014)]{simonyan2014very}
Karen Simonyan and Andrew Zisserman.
\newblock Very deep convolutional networks for large-scale image recognition.
\newblock \emph{arXiv preprint arXiv:1409.1556}, 2014.

\bibitem[Streeter and McMahan(2010)]{lessregret}
Matthew Streeter and H.~Brendan McMahan.
\newblock Less regret via online conditioning.
\newblock \emph{arXiv preprint arXiv:1002.4862}, 2010.

\bibitem[Ward et~al.(2019)Ward, Wu, and Bottou]{ward2019adagrad}
Rachel Ward, Xiaoxia Wu, and Leon Bottou.
\newblock Adagrad stepsizes: sharp convergence over nonconvex landscapes.
\newblock In \emph{International Conference on Machine Learning}, 2019.

\bibitem[Weston and Watkins(1999)]{multimargin}
Jason Weston and Christopher Watkins.
\newblock Support vector machines for multi-class pattern recognition.
\newblock pages 219--224, 01 1999.

\bibitem[Zbontar et~al.(2018)Zbontar, Knoll, Sriram, Muckley, Bruno, Defazio,
  Parente, Geras, Katsnelson, Chandarana, et~al.]{zbontar2018fastmri}
Jure Zbontar, Florian Knoll, Anuroop Sriram, Matthew~J. Muckley, Mary Bruno,
  Aaron Defazio, Marc Parente, Krzysztof~J. Geras, Joe Katsnelson, Hersh
  Chandarana, et~al.
\newblock {fastMRI}: An open dataset and benchmarks for accelerated {MRI}.
\newblock \emph{arXiv preprint arXiv:1811.08839}, 2018.

\bibitem[Zhang et~al.(2022)Zhang, Cutkosky, and Paschalidis]{pdecoin}
Zhiyu Zhang, Ashok Cutkosky, and Ioannis~Ch. Paschalidis.
\newblock {PDE}-based optimal strategy for unconstrained online learning.
\newblock In \emph{Proceedings of the 39th International Conference on Machine
  Learning (ICML 2022)}, 2022.

\bibitem[Zoph and Le(2017)]{zoph2017neural}
Barret Zoph and Quoc Le.
\newblock Neural architecture search with reinforcement learning.
\newblock In \emph{International Conference on Learning Representations}, 2017.
\newblock URL \url{https://openreview.net/forum?id=r1Ue8Hcxg}.

\end{thebibliography}
\bibliographystyle{plainnat}
\clearpage

\appendix
\section{Analysis of Prodigy}
As a reminder, we use the notation $\log_{2+}(a) = 1 + \log_2(a)$ to denote the logarithm that is lower bounded by $1$ for any $a\ge 1$.
\subsection{Useful propositions}
\begin{prop}[Lemma A.2 in \cite{levy2018online}]\label{pr:da_sequence_bound} For any sequence of nonnegative real numbers $a_0,\dotsc, a_n$
\begin{equation}
\sqrt{\sum_{k=0}^{n}a_i}
\le \sum_{k=0}^{n}\frac{a_k}{\sqrt{\sum_{i=0}^{k}a_i}}\leq 2\sqrt{\sum_{k=0}^{n}a_i}. \label{eq:streeter_mchmahan_modified}
\end{equation}
\end{prop}
\begin{proof}
    For completeness, we prove both statements here.
    Notice that for any $\alpha\in[0, 1]$, it holds $1 - \sqrt{1-\alpha} \le \alpha \le 2(1 - \sqrt{1-\alpha})$. Substituting $\alpha=\frac{a_k}{\sum_{i=0}^k a_i}$ gives
    \[
        1 - \sqrt{1-\frac{a_k}{\sum_{i=0}^k a_i}} \le \frac{a_k}{\sum_{i=0}^k a_i} \le 2\left(1 - \sqrt{1-\frac{a_k}{\sum_{i=0}^k a_i}} \right).
    \]
    If we multiply all sides by $\sqrt{\sum_{i=0}^k a_i}$, the inequality above becomes
    \[
        \sqrt{\sum_{i=0}^k a_i} - \sqrt{{\sum_{i=0}^{k-1} a_i}} \le \frac{a_k}{\sqrt{\sum_{i=0}^k a_i}} \le 2\left(\sqrt{\sum_{i=0}^k a_i} - \sqrt{{\sum_{i=0}^{k-1} a_i}}\right).
    \]
    Summing over $k = 0,\dotsc, n$, we get the stated bound.
\end{proof}

\begin{prop}\label{pr:err-corr} 
    For any sequence of nonnegative numbers $a_0, \dotsc, a_n$ and $A>0$, it holds
\begin{equation}\label{eq:gd_sequence}
    \sum_{k=0}^{n}\frac{a_k}{A + \sum_{i=0}^{k}a_i}
    \le \log\biggl(A+\sum_{k=0}^{n}a_k\biggr) - \log(A).
\end{equation}
\end{prop}
\begin{proof}
    If $a_i=0$ for some $i$, we can simply ignore the corresponding summands, so let us assume that $a_i>0$ for all $i$. For any $t>0$ it holds $1/(1+t) \le \log(1+1/t)$. Substituting $t=S_{k}/a_k$, where $S_{k} = A + \sum_{i=0}^{k-1} a_i$ for $k>0$ and $S_0=A$, we get
    \[
        \frac{1}{1+\frac{S_{k}}{a_k}}
        = \frac{a_{k}}{a_k + S_{k}}
        = \frac{a_{k}}{A + \sum_{i=0}^k a_i}
        \le \log (1 + a_k/S_{k})
        = \log(S_{k+1}) - \log(S_{k}).
    \]
    Summing this over $k$ from $0$ to $n$, we get
    \begin{align*}
        \sum_{k=0}^{n}\frac{a_k}{A + \sum_{i=0}^{k}a_i}
        &\le \sum_{k=0}^n \left(\log(S_{k+1}) - \log(S_{k})\right)
        = \log(S_{n+1}) - \log(S_{0}) \\
        &= \log\biggl(A+\sum_{k=0}^{n}a_k\biggr) - \log(A).
    \end{align*}
    This is exactly what we wanted to prove.
\end{proof}
\subsection{Proof of Lemma~\ref{lem:d_sequence}}
\begin{proof} Following the proof in \cite{dog},
we define $K=\bigl\lceil\log_{2}\bigl(\frac{d_{N}}{d_{0}}\bigr)\bigr\rceil$
and $n=\bigl\lfloor\frac{N}{K}\bigr\rfloor$. Consider a partitioning
of the sequence $t\leq N$ into half-open intervals $I_{k}=\left[nk,n(k+1)\right)$
for $k=0$ to $K-1$. We want to show that there is at least one interval
such that $d_k$ changes by at most a factor of 2 on that interval. We
will use proof by contradiction.

Suppose that for all intervals, $d_{nk}<\frac{1}{2}d_{n(k+1)}$. Then
$d_k$ at least doubles in every interval, and so:
\[
d_{0}<\frac{1}{2}d_{n}<\frac{1}{4}d_{2n}\dots<\frac{1}{2^{K}}d_{nK}<\frac{1}{2^{K}}d_{N},
\]
which implies that $d_{N}/d_{0}>2^{K}$ and so $K<\log_{2}\left(d_{N}/d_{0}\right)$
which contradicts our definition $K=\bigl\lceil\log_{2}\bigl(\frac{d_{N}}{d_{0}}\bigr)\bigr\rceil$.
Therefore, there exists some $\hat{k}$ such that $d_{n\hat{k}}\geq\frac{1}{2}d_{n(\hat{k}+1)}$. We can now proceed with proving
the Lemma by considering the summation over interval $I_{\hat{k}}$
only:
\begin{align*}
\min_{t<N}\frac{d_{t+1}}{\sqrt{\sum_{k=0}^{t}d_{k}^{2}}} & \le\frac{d_{n(\hat{k}+1)}}{\sqrt{\sum_{k=0}^{n(\hat{k}+1)-1}d_{k}^{2}}}\le\frac{d_{n(\hat{k}+1)}}{\sqrt{\sum_{k=n\hat{k}}^{n(\hat{k}+1)-1}d_{k}^{2}}}\leq\frac{d_{n(\hat{k}+1)}}{\sqrt{\sum_{k=n\hat{k}}^{n(\hat{k}+1)-1}d_{n\hat{k}}^{2}}}\\
 & =\frac{d_{n(\hat{k}+1)}}{\sqrt{nd_{n\hat{k}}^{2}}}\leq\frac{d_{n(\hat{k}+1)}}{\sqrt{\frac{1}{4}nd_{n\left(\hat{k}+1\right)}^{2}}}=\frac{2}{\sqrt{n}}=\frac{2}{\sqrt{\bigl\lfloor\frac{N}{K}\bigr\rfloor}}\\
 & \le\frac{2}{\sqrt{\frac{N}{K}-1}} \le \frac{2}{\sqrt{\frac{N}{\log_{2}(d_{N}/d_{0})+1}-1}}=\frac{2\sqrt{\log_{2+}\bigl(\frac{d_{N}}{d_{0}}\bigr)}}{\sqrt{N-\log_{2+}\bigl(\frac{d_{N}}{d_{0}}\bigr)}}\\
 & \hspace{-8.5mm}\overset{N\ge2\log_{2+}(\frac{d_{N}}{d_{0}})}{\le}\frac{4\sqrt{\log_{2+}\bigl(\frac{d_{N}}{d_{0}}\bigr)}}{\sqrt{N}}.
\end{align*}
\end{proof}
\subsection{GD Analysis}
\begin{lem}\label{lem:d_lower_bound}
    Assume that $d_0\le D$. Then, the estimate $d_k$ in Algorithm~\ref{alg:dadagradv2gd} satisfies $d_k\le D$ for all $k$.
\end{lem}
\begin{proof}
    By optimality of $f_*$, we have $f(x_k) - f_*\ge 0$, so
    \begin{align*}
        0
        \le \sum_{k=0}^n \eta_k (f(x_k) - f_*) 
        \le \sum_{k=0}^n \eta_k \langle g_k, x_k - x_*\rangle
        = \sum_{k=0}^n \eta_k \langle g_k, x_0 - x_*\rangle + \sum_{k=0}^n \eta_k \langle g_k, x_k - x_0\rangle.
    \end{align*}
    Collecting the gradients in the first sum together and using Cauchy-Schwarz inequality, we obtain
    \begin{align}
        0
        &\le \sum_{k=0}^n \eta_k (f(x_k) - f_*) 
        \le \langle x_0 - x_{n+1}, x_0 - x_*\rangle+ \sum_{k=0}^n \eta_k \langle g_k, x_k - x_0\rangle \notag\\
        &\le \| x_0 - x_{n+1}\|\|x_0 - x_*\| + \sum_{k=0}^n \eta_k \langle g_k, x_k - x_0\rangle . \label{eq:upper_bound_func_vals_gd}
    \end{align}
    Using the definition of $\hat d_{n+1}$, this is equivalent to $0\le (D - \hat d_{n+1})\|x_0 - x_{n+1}\|$, which implies $\hat d_{n+1}\le D$. Therefore, since $d_0\le D$, we can show by induction $ d_{n+1}\le D$ as well.
\end{proof}
\begin{lem}\label{lem:gd_s_to_d}
    The following inequality holds for the iterates of Algorithm~\ref{alg:dadagradv2gd}:
    \begin{align*}
        \|x_{n+1} - x_{0}\|
        \le 2d_{n+1} + \frac{1}{2d_{n+1}}\sum_{k=0}^{n}\eta_{k}^{2}\|g_k\|^2.
    \end{align*}
\end{lem}
\begin{proof}
    Let us rewrite $\hat d_{n+1}$ in a slightly different manner:
    \begin{align*}
        \hat{d}_{n+1}\|x_{n+1} - x_0\|
        &\overset{\mathrm{def}}{=} \sum_{k=0}^n \langle x_k - x_{k+1}, x_0 - x_k \rangle  \\
        &= \sum_{k=0}^n \frac{1}{2}\left(\|x_{k+1} - x_0\|^2 - \| x_k - x_{k+1}\|^2 - \| x_k - x_0 \|^2 \right) \\
        &= \frac{1}{2}\|x_{n+1} - x_0\|^2 -\frac{1}{2} \sum_{k=0}^n \|x_k - x_{k+1}\|^2 .
    \end{align*}
    Combining this with the property $\hat{d}_{n+1}\le d_{n+1}$, we derive
    \[
    \frac{1}{2}\left\Vert x_{n+1} - x_0\right\Vert ^{2}-\frac{1}{2}\sum_{k=0}^{n}\left\Vert x_{k} - x_{k+1}\right\Vert ^{2} 
    = \hat{d}_{n+1}\left\Vert x_{n+1} - x_0\right\Vert \le d_{n+1}\left\Vert x_{n+1} - x_0\right\Vert.
    \]
    Applying inequality $2\alpha\beta\le \alpha^2 + \beta^2$ with $\alpha^2 = 2d_{n+1}^2$ and $\beta^2= \frac{1}{2}\|x_{n+1} - x_0\|^2$ and plugging-in the bound above, we establish
    \begin{align*}
        2d_{n+1} \|x_{n+1} - x_0\|
        &=2\alpha\beta 
        \le \alpha^2 + \beta^2
        = 2d_{n+1}^2 + \frac{1}{2}\|x_{n+1} - x_0\|^2 \\
        &\le 2d_{n+1}^2 + d_{n+1}\|x_{n+1} - x_0\| + \frac{1}{2}\sum_{k=0}^{n}\|x_{k} - x_{k+1}\|^2.
    \end{align*}
    Rearranging the terms, we obtain
    \begin{align*}
        d_{n+1} \|x_{n+1} - x_0\|
        &\le 2d_{n+1}^2 + \frac{1}{2}\sum_{k=0}^{n}\|x_k- x_{k+1}\|^2
        =  2d_{n+1}^2 + \frac{1}{2}\sum_{k=0}^{n}\eta_{k}^{2}\|g_k\|^2.
    \end{align*}
    It remains to divide this inequality by $d_{n+1}$ to get the desired claim.
\end{proof}

\begin{lem}
    Assuming the weights $\lambda_0, \dotsc, \lambda_n$ are positive, it holds for the iterates of Algorithm~\ref{alg:dadagradv2gd}:
    \begin{equation}\label{eq:gd_stepsize_lemma}
        \sum_{k=0}^n \frac{d_k^4\lambda_k^2\|g_k\|^2}{d_k^2G^2 + \sum_{i=0}^kd_i^2\lambda_i^2\|g_i\|^2}
        \le d_{n}^2\log\left(1 + \sum_{k=0}^{n}\lambda_k^2\right).
    \end{equation}
\end{lem}
\begin{proof}
    The lemma follows straightforwardly from Proposition~\ref{pr:err-corr} by substituting $a_{k} = \frac{d_k^2}{d_{n}^2}\lambda_k^2\|g_k\|^2$ for $k$ from 0 to $n$:
    \begin{align*}
        \sum_{k=0}^n \frac{d_k^4\lambda_k^2\|g_k\|^2}{d_k^2G^2 + \sum_{i=0}^kd_i^2\lambda_i^2\|g_i\|^2}
        &= d_{n}^2\sum_{k=0}^n \frac{\frac{d_k^2}{d_{n}^2} \lambda_k^2\|g_k\|^2}{G^2 + \sum_{i=0}^k\frac{d_i^2}{d_k^2}\lambda_i^2\|g_i\|^2} \\
        &\hspace{-2.5mm}\overset{d_k\le d_{n}}{\le} d_{n}^2\sum_{k=0}^n \frac{\frac{d_k^2}{d_{n}^2} \lambda_k^2\|g_k\|^2}{G^2 + \sum_{i=0}^k\frac{d_i^2}{d_n^2}\lambda_i^2\|g_i\|^2} \\
        &\overset{\eqref{eq:gd_sequence}}{\le }d_{n}^2\left(\log\left( G^2 + \sum_{k=0}^{n}\frac{d_k^2}{d_{n}^2}\lambda_k^2\|g_k\|^2\right) - \log(G^2)\right) \\
        &\le d_{n}^2\log\left(1 + \sum_{k=0}^{n}\lambda_k^2\right),
    \end{align*}
    where in the last step we used $\frac{d_k^2}{d_{n}^2}\lambda_k^2\|g_k\|^2\le \lambda_k^2 G^2$.
\end{proof}
Let us restate Theorem~\ref{thm:gd}:
\begin{thm}[Same as Theorem~\ref{thm:gd}]
    Given any weights $1\le\lambda_0\le\dotsb \lambda_n$, the functional gap of the average iterate of Algorithm~\ref{alg:dadagradv2gd} converges as
    \[
        f(\hat x_n) - f_*
        \le \sqrt{2\lambda_{n}}DG\frac{2d_{n+1} + d_{n+1}\log(1+\sum_{k=0}^n \lambda_k^2)}{\sqrt{\sum_{k=0}^n \lambda_k d_k^2}}.
    \]
\end{thm}
\begin{proof}
    The first steps in the proof follow the same lines as the theory in~\cite{defazaio2023learning}, but we still provide them for completeness.
    
    Firstly, let us continue developing the bound proved in the proof of Lemma~\ref{lem:d_lower_bound}:
    \begin{align*}
        \sum_{k=0}^n \eta_k (f(x_k) - f_*) 
        &\le \|x_0 - x_{n+1}\|D +  \sum_{k=0}^n \eta_k \langle g_k, x_k - x_0\rangle \\
        &= \|x_0 - x_{n+1}\|D +  \sum_{k=0}^n \langle x_k - x_{k+1}, x_k - x_0\rangle \\
        &= \|x_0 - x_{n+1}\|D +  \frac{1}{2}\sum_{k=0}^n \left[\|x_k - x_{k+1}\|^2 +  \|x_k - x_0\|^2 - \|x_{k+1} - x_0\|^2\right] \\
        &\le \|x_0 - x_{n+1}\|D + \frac{1}{2}\sum_{k=0}^n \|x_k - x_{k+1}\|^2. 
    \end{align*}
    We upper bound the first term with the help of Lemma~\ref{lem:gd_s_to_d}:
    \begin{align*}
        \sum_{k=0}^n \eta_k (f(x_k) - f_*) 
        &\le 2Dd_{n+1} + \frac{D}{2d_{n+1}}\sum_{k=0}^{n}\eta_{k}^{2}\|g_k\|^2 + \frac{1}{2}\sum_{k=0}^n \eta_k^2\|g_k\|^2.
    \end{align*}
    Since by Lemma~\ref{lem:d_lower_bound}, $1\le \frac{D}{d_{n+1}}$, we can simplify it to
    \begin{align*}
        \sum_{k=0}^n \eta_k (f(x_k) - f_*) 
        &\le 2Dd_{n+1} + \frac{D}{d_{n+1}}\sum_{k=0}^{n}\eta_{k}^{2}\|g_k\|^2 \\
        &= 2Dd_{n+1} + \frac{D}{d_{n+1}}\sum_{k=0}^{n}\frac{d_k^4\lambda_k^2}{d_k^2G^2 + \sum_{i=0}^k d_i^2\lambda_i^2\|g_i\|^2}\|g_k\|^2 \\
        &\overset{\eqref{eq:gd_stepsize_lemma}}{\le} 2Dd_{n+1} + \frac{D}{d_{n+1}}d_{n}^2\log\Bigl(1+\sum_{k=0}^n\lambda_k^2\Bigr).
    \end{align*}
    Using the convexity of $f$, we can apply Jensen's inequality on the iterate $\hat x_n$ to get
    \begin{align}
        f(\hat x_n) - f_*
        &\le \frac{1}{\sum_{k=0}^n \eta_k}\sum_{k=0}^n\eta_k(f(x_k) - f_*)
        \le \frac{2Dd_{n+1} + \frac{D}{d_{n+1}}d_{n}^2\log(1+\sum_{k=0}^n\lambda_k^2)}{\sum_{k=0}^n \eta_k} \notag \\
        &\le D\frac{2d_{n+1} + d_{n+1}\log(1+\sum_{k=0}^n\lambda_k^2)}{\sum_{k=0}^n \eta_k}. \label{eq:average_gd_bound}
    \end{align}
    Notice that $\|g_i\|\le G$ and $\lambda_i \le \lambda_n$, so
    \[
        \eta_k 
        = \frac{d_k^2\lambda_{k}}{\sqrt{d_k^2G^2 + \sum_{i=0}^{k}d_i^2\lambda_i^2\left\Vert g_{i}\right\Vert ^{2}}}
        \ge \frac{d_k^2\lambda_{k}}{G\sqrt{d_k^2 + \sum_{i=0}^{k}d_i^2\lambda_i^2}}
        \ge \frac{d_k^2\lambda_k}{G\sqrt{2\lambda_n}\sqrt{\sum_{i=0}^kd_i^2\lambda_i}}.
    \]
    Summing over $k$ from $0$ to $n$ gives
    \[
        \sum_{k=0}^n \eta_k \ge \frac{1}{\sqrt{2\lambda_n}G}\sum_{k=0}^n \frac{d_k^2\lambda_k}{\sqrt{\sum_{i=0}^kd_i^2\lambda_i}}
        \overset{\eqref{eq:streeter_mchmahan_modified}}{\ge}
        \frac{1}{\sqrt{2\lambda_n}G}\sqrt{\sum_{k=0}^n d_k^2\lambda_k}.
    \]
    Hence,
    \[
        f(\hat x_n) - f_*
        \overset{\eqref{eq:average_gd_bound}}{\le} \sqrt{2\lambda_n}DG\frac{d_{n+1}}{\sqrt{\sum_{k=0}^n d_k^2\lambda_k}}\left(2+ \log\left(1+\sum_{k=0}^n\lambda_k^2\right)\right).
    \]
\end{proof}
\begin{corollary}
    Consider Algorithm~\ref{alg:dadagradv2gd} with $n\ge 2\log_{2}\left(\frac{2D}{d_0}\right)$ and define $t = \arg\min_{k\le n} \frac{d_k}{\sqrt{\sum_{i=0}^k d_i^2}}$. If we choose weights $\lambda_k = 1$, then it holds
    \[
        f(\hat x_t) - f_*
        \le 4\sqrt{2}DG\frac{2 + \log(n+2)}{\sqrt{n}}\sqrt{\log_{2}\left(\frac{2D}{d_0}\right)}.
    \]
\end{corollary}
\begin{proof}
    Substituting $\lambda_k$ in the bound of Theorem~\ref{thm:gd}, we get for any $n$
    \[
        f(\hat x_n) - f_*
        \overset{\eqref{eq:average_gd_bound}}{\le} \sqrt{2}DG\frac{d_{n+1}}{\sqrt{\sum_{k=0}^n d_k^2}}\log\left(n + 2\right).
    \]
    Using the definition of $t$, the result of Lemma~\ref{lem:d_sequence} and the property $d_n\le D$, we obtain
    \begin{align*}
        f(\hat x_t) - f_*
        &\le \sqrt{2}DG\min_{k\le n}\frac{d_{k+1}}{\sqrt{\sum_{i=0}^k d_i^2}}\left(2+\log\left(n + 2\right)\right) \\
        &\le 4\sqrt{2}DG\frac{2 + \log(n+2)}{\sqrt{n}}\sqrt{\log_{2}\left(\frac{2D}{d_0}\right)}.
    \end{align*}
\end{proof}
\begin{corollary}\label{cor:asymptotic}
    Choose any $p\ge 0$ ans set the weights to be $\lambda_k = (k+1)^p$. Then,
    \[
        f(\hat x_n) - f_* = \mathcal{O}\left(\frac{DG(p+1)^{3/2}\log(n+1)}{\sqrt{n+1}}\right).
    \]
\end{corollary}
\begin{proof}
    Since the sequence $d_0,d_1,\dotsc$ is non-decreasing and upper bounded by $D$, there exists an index $\hat n$ such that $d_{k}\le 2 d_{\hat n}$ for any $k\ge \hat n$. Moreover, we have for $n\ge 2(\hat n + 1)$
    \[
        \sum_{k=\hat n}^n \lambda_k 
        \ge \frac{1}{p+1}\left((n+1)^{p+1} - (\hat n+1)^{p+1} \right)
        \ge \frac{1}{2(p+1)}(n+1)^{p+1}
    \]
    and 
    \[
        \sum_{k=0}^n \lambda_k^2 
        = \sum_{k=1}^{n+1} k^{2p}
        \le \int_2^{n+2}x^{2p}dx
        \le \frac{1}{2p+1}(n+2)^{2p+1} - 1
        \le (n+2)^{2p+1} - 1.
    \]
    Let us plug this into the bound of Theorem~\ref{thm:gd} for $n\ge 2(\hat n + 1)$:
    \begin{align*}
        f(\hat x_n) - f_*
        &\le \sqrt{2\lambda_n}DG\frac{d_{n+1}}{\sqrt{\sum_{k=0}^n d_k^2\lambda_k}}\left(2+ \log\left(1+\sum_{k=0}^n\lambda_k^2\right)\right) \\
        &\le \frac{2d_{\hat n}\sqrt{2(n+1)^p}DG}{\sqrt{d_{\hat n}^2\sum_{k=\hat n}^n\lambda_k}}\left(2 + (2p+1)\log(n+2)\right) \\
        &\le \frac{4\sqrt{p+1}DG}{\sqrt{n+1}}\left(2 + (2p+1)\log(n+2)\right) = \mathcal{O}\left(\frac{DG(p+1)^{3/2}\log(n+1)}{\sqrt{n+1}}\right),
    \end{align*}
    which matches our claim.
\end{proof}
Notice that the bound in Corollary~\ref{cor:asymptotic} does not depend on $D/d_0$. This is only possible asymptotically for a large enough $k$ and a similar bound without weights was presented by \citet{defazaio2023learning}.
\subsection{DA Analysis}\label{sec:da}
\begin{lem}\label{lem:da_s}
Considering Algorithm~\ref{alg:dadagradv2da}, we have
\begin{equation*}
    \left\Vert s_{n+1}\right\Vert 
    \leq\frac{2d_{n+1}}{\gamma_{n+1}} + \frac{\sum_{k=0}^{n}\gamma_{k}\lambda_{k}^{2}\|g_k\|^2}{2d_{n+1}}.
\end{equation*}
\end{lem}
\begin{proof}
    When studying Dual Averaging, we need to introduce an extra sequence that lower bounds $\overline d_n$:
    \[
        \overline{d}_{n+1} 
        \overset{\mathrm{def}}{=} \frac{\gamma_{n+1}\left\Vert s_{n+1} \right\Vert ^{2}-\sum_{k=0}^{n}\gamma_k\lambda_k^2\left\Vert g_{k}\right\Vert ^{2} }{2\|s_{n+1}\|}.
    \]
    Let us show that $\hat d_{n+1}\ge \overline{d}_{n+1}$ by comparing their numerators:
    \begin{align*}
        \hat d_{n+1}\|s_{n+1}\|
        &=\sum_{k=0}^n \lambda_k\langle g_k, x_0 - x_k\rangle
        = \sum_{k=0}^n \lambda_k\gamma_k\langle g_k,s_k\rangle
        = \sum_{k=0}^n \gamma_k\langle s_{k+1}-s_k,s_k\rangle \\
        &= \sum_{k=0}^n \frac{\gamma_k}{2}\left[\|s_{k+1}\|^2 - \|s_{k+1}-s_k\|^2 - \|s_k\|^2\right] \\
        &= \frac{\gamma_{n}}{2}\|s_{n+1}\|^2+ \frac{1}{2}\sum_{k=0}^n (\gamma_{k} - \gamma_{k+1})\|s_{k+1}\|^2-\frac{1}{2}\sum_{k=0}^n \gamma_k\lambda_k^2\|g_k\|^2  \\
        &\hspace{-4.3mm}\overset{\gamma_k\ge \gamma_{k+1}}{\ge} \frac{\gamma_{n+1}}{2}\|s_{n+1}\|^2 - \frac{1}{2}\sum_{k=0}^n \gamma_k\lambda_k^2\|g_k\|^2 \\
        &= \overline{d}_{n+1} \|s_{n+1}\|.
    \end{align*}
    Using the definition of $\overline{d}_{n+1}$,
    and the property $\overline{d}_{n+1}\le \hat{d}_{n+1}\le d_{n+1}$, we derive
    \[
    \frac{\gamma_{n+1}}{2}\left\Vert s_{n+1}\right\Vert ^{2}-\frac{1}{2}\sum_{k=0}^{n}\gamma_k\lambda_k^2\left\Vert g_{k} \right\Vert ^{2} 
    = \overline{d}_{n+1}\left\Vert s_{n+1}\right\Vert \le d_{n+1}\left\Vert s_{n+1}\right\Vert.
    \]
    Using inequality $2\alpha\beta\le \alpha^2 + \beta^2$ with $\alpha^2 = \frac{2d_{n+1}^2}{\gamma_{n+1}}$ and $\beta^2= \frac{\gamma_{n+1}}{2}\|s_{n+1}\|^2$ and then the bound above, we establish
    \begin{align*}
        2d_{n+1} \|s_{n+1}\|
        &= 2\alpha\beta 
        \le \alpha^2+\beta^2 
        = \frac{2d_{n+1}^2}{\gamma_{n+1}} + \frac{\gamma_{n+1}}{2}\|s_{n+1}\|^2 \\
        &\le \frac{2d_{n+1}^2}{\gamma_{n+1}} + d_{n+1}\|s_{n+1} \| + \frac{1}{2}\sum_{k=0}^{n}\gamma_k\lambda_k^2\|g_{k}\|^2.
    \end{align*}
    Rearranging the terms, we obtain
    \begin{align*}
        d_{n+1} \|s_{n+1}\|
        &\le \frac{2d_{n+1}^2}{\gamma_{n+1}} + \frac{1}{2}\sum_{k=0}^{n}\gamma_k\lambda_k^2\|g_{k}\|^2.
    \end{align*}
    It remains to divide both sides by $d_{n+1}$.
\end{proof}

\begin{lem}\label{lem:da_sum_f_bounded_by_ds}
    The Dual Averaging algorithm (Algorithm~\ref{alg:dadagradv2da}) satisfies
    \begin{equation}
        \sum_{k=0}^n\lambda_k (f(x_k) - f_*)
        \le (D - \hat d_{n+1})\|s_{n+1}\|. \label{eq:da_sum_f_bounded_by_ds}
    \end{equation}
\end{lem}
\begin{proof}
    Summing inequality $f(x_k) - f_*\le \langle g_k, x_k - x_*\rangle $ with weights $\lambda_k$, we get
    \begin{align*}
        \sum_{k=0}^n \lambda_k (f(x_k) - f_*) 
        \le \sum_{k=0}^n \lambda_k \langle g_k, x_k - x_*\rangle
        = \sum_{k=0}^n \lambda_k \langle g_k, x_0 - x_*\rangle + \sum_{k=0}^n \lambda_k \langle g_k, x_k - x_0\rangle.
    \end{align*}
    Using Cauchy-Schwarz on the first product in the right-hand side and then telescoping the second sum, we obtain
    \begin{align*}
        \sum_{k=0}^n \lambda_k (f(x_k) - f_*) 
        &\le \| s_{n+1}\| \|x_0 - x_*\| + \sum_{k=0}^n \lambda_k\langle g_k, x_k - x_0\rangle  \\
        &= \| s_{n+1}\|D - \hat d_{n+1}\|s_{n+1}\|.
    \end{align*}
\end{proof}
Next, we restate and prove Theorem~\ref{thm:da}:
\begin{thm}[Same as Theorem~\ref{thm:da}]
    For Algorithm~\ref{alg:dadagradv2da}, it holds that:
    \[
        f(\overline x_t) - f_*
        \le \frac{4GD}{\sqrt{n}}\sqrt{\log_{2}\Bigl(\frac{2D}{d_0}\Bigr)},
    \]
    where $t=\arg\min_{k\le n} \frac{d_{k+1}}{\sqrt{\sum_{i=0}^k d_i^2}}$.
\end{thm}
\begin{proof}
    Let us sum inequality $\lambda_k(f(x_k) - f_*)\ge 0$ and then apply Lemma~\ref{lem:da_sum_f_bounded_by_ds}:
    \begin{align*}
        0
        &\le \sum_{k=0}^n \lambda_k (f(x_k) - f_*) \overset{\eqref{eq:da_sum_f_bounded_by_ds}}{\le} (D - \hat d_{n+1})\| s_{n+1}\|.
    \end{align*}
    Clearly, this implies that $\hat d_{n+1}\le D$, and by induction it follows that $d_{n+1}\le D$ as well. Now let us upper bound the functional values:
    \begin{align*}
        \sum_{k=0}^n \lambda_k (f(x_k) - f_*)
        &\overset{\eqref{eq:da_sum_f_bounded_by_ds}}{\le} D\|s_{n+1}\| - \sum_{k=0}^n \gamma_{k}\lambda_k\langle g_k, s_{k}\rangle \\
        & = D\| s_{n+1}\| - \sum_{k=0}^n \gamma_{k}\langle s_{k+1} - s_k, s_{k}\rangle  \\
        &= D\| s_{n+1}\| + \frac{1}{2}\sum_{k=0}^n \gamma_{k}\left(\|s_{k+1} - s_k\|^2 + \|s_{k}\|^2 - \|s_{k+1}\|^2\right) \\
        &= D\| s_{n+1}\| + \frac{1}{2}\sum_{k=0}^n \gamma_{k}\|s_{k+1} - s_k\|^2 + \frac{1}{2}\sum_{k=0}^n (\gamma_{k} - \gamma_{k-1})\|s_{k}\|^2 - \frac{\gamma_{n}}{2}\|s_{n+1}\|^2 .
    \end{align*}
    We can drop the last two terms since $\gamma_k\le \gamma_{k-1}$:
    \begin{align*}
        \sum_{k=0}^n \lambda_k (f(x_k) - f_*) 
        &\le D\| s_{n+1}\| + \frac{1}{2}\sum_{k=0}^n \gamma_{k}\|s_{k+1} - s_k\|^2 \\
        &= D\| s_{n+1}\| + \frac{1}{2}\sum_{k=0}^n \gamma_{k}\lambda_k^2\|g_k\|^2.
    \end{align*}
    The first term in the right-hand side is readily bounded by Lemma~\ref{lem:da_s}:
    \begin{align*}
        \sum_{k=0}^n \lambda_k (f(x_k) - f_*) 
        &\le D\| s_{n+1}\| + \frac{1}{2}\sum_{k=0}^n \gamma_{k}\lambda_k^2\|g_k\|^2 \\
        &\le \frac{2Dd_{n+1}}{\gamma_{n+1}} + \frac{D}{2d_{n+1}} \sum_{k=0}^{n}\gamma_{k}\lambda_{k}^{2}\|g_k\|^2 + \frac{1}{2}\sum_{k=0}^n \gamma_{k}\lambda_k^2\|g_k\|^2 \\
        &\hspace{-4mm}\overset{d_{n+1}\le D}{\le} \frac{2Dd_{n+1}}{\gamma_{n+1}} + \frac{D}{d_{n+1}} \sum_{k=0}^{n}\gamma_{k}\lambda_{k}^{2}\|g_k\|^2 \\
        &\hspace{-3mm}\overset{\lambda_k\le \lambda_n}{\le}\frac{2Dd_{n+1}}{\gamma_{n+1}} + \frac{D}{d_{n+1}} \lambda_n\sum_{k=0}^{n}\gamma_{k}\lambda_{k}\|g_k\|^2.
    \end{align*}
    Then, apply Proposition~\ref{pr:da_sequence_bound}:
    \begin{align*}
        \sum_{k=0}^n \lambda_k (f(x_k) - f_*) 
        &\le \frac{2D}{\gamma_{n+1}} + \frac{D}{d_{n+1}} \lambda_n\sum_{k=0}^{n}\gamma_{k}\lambda_{k}\|g_k\|^2 \\
        &=\frac{2D}{\gamma_{n+1}} + \frac{D}{d_{n+1}} \lambda_n\sum_{k=0}^{n}\frac{1}{\sqrt{\lambda_kG^2 + \sum_{i=0}^{k-1}\lambda_i\|g_i\|^2}}\lambda_{k}\|g_k\|^2 \\
        &\le \frac{2D}{\gamma_{n+1}} + \frac{D}{d_{n+1}} \lambda_n\sum_{k=0}^{n}\frac{1}{\sqrt{\lambda_k \|g_k\|^2 + \sum_{i=0}^{k-1}\lambda_i\|g_i\|^2}}\lambda_{k}\|g_k\|^2 \\
        &\overset{\eqref{eq:streeter_mchmahan_modified}}{\le} \frac{2D}{\gamma_{n+1}} +  \frac{2D}{d_{n+1}} \lambda_n\sqrt{\sum_{k=0}^{n}\lambda_k\|g_k\|^2}.
    \end{align*}
    Let us now plug-in $\lambda_k=d_k^2$ and bound each gradient norm using $\|g_k\|\le G$:
    \begin{align*}
        \sum_{k=0}^n \lambda_k (f(x_k) - f_*)
        \le 4D d_{n+1}\sqrt{\sum_{k=0}^n d_k^2\|g_k\|^2}
        \le 4GD d_{n+1}\sqrt{\sum_{k=0}^n d_k^2}.
    \end{align*}
    Thus, we get the following convergence rate:
    \begin{align*}
        f(\overline x_t) - f_*
        &\le \frac{4GD d_{t+1}\sqrt{\sum_{k=0}^t d_k^2}}{\sum_{k=0}^t d_k^2}
        = \frac{4GD d_{t+1}}{\sqrt{\sum_{k=0}^t d_k^2}}
        = \min_{t^\prime<n}\frac{4GD d_{t^\prime+1}}{\sqrt{\sum_{k=0}^{t^\prime} d_k^2}} \\
        &\le \frac{4GD}{\sqrt{n}}\sqrt{\log_{2+}\Bigl(\frac{D}{d_0}\Bigr)}.
    \end{align*}
\end{proof}

\subsection{Coordinate-wise Prodigy}\label{sec:coordinate}
\begin{algorithm}[t]
\begin{algorithmic}[1]
    \State {\bfseries Input:} $d_0>0$, $x_0$,  $G_{\infty}\ge 0$; $s_{0} = 0\in\mathbb{R}^p$, $a_0 = 0\in\mathbb{R}^p$
    \For{$k=0$ {\bfseries to} $n$}
    
    \State $g_{k} \in \partial f(x_{k})$
        \State $\lambda_k = d_k^2$
        \State $s_{k+1} = s_k + \lambda_k g_k$
	\State $\hat{d}_{k+1}=\dfrac{\sum_{i=0}^k \lambda_i\langle g_i, x_0 - x_i\rangle}{\left\Vert s_{k+1}\right\Vert_1 }$
        \State $d_{k+1}= \max(d_k, \hat d_{k+1})$
        \State $a_{k+1}^2 = \lambda_{k+1}G_{\infty}^2 + \sum_{i=0}^k\lambda_i g_i^2$ \hfill$\triangleright$ \text{Coordinate-wise square}
        \State $\mA_{k+1} = \diag(a_{k+1})$
	\State $x_{k+1}=x_0-\mA_{k+1}^{-1}s_{k+1}$
    \EndFor
	\State Return $\bar{x}_{n}=\frac{1}{n+1}\sum ^n_{k=0}\lambda_k x_k$
\end{algorithmic}
\caption{\label{alg:prodigy_coord}Prodigy (Coordinate-wise Dual Averaging version)}
\end{algorithm}
Here we study Algorithm~\ref{alg:prodigy_coord}. The theory in this section follows closely the analysis in Section~\ref{sec:da}. There are only a few minor differences such as the use of weighted norms, which we define as $\langle x, y\rangle_{\mA^{-1}} = x^\top \mA^{-1}y$ for any matrix $\mA\succcurlyeq 0$. In addition, we use $\ell_{\infty}$ norm for the distance term and for the gradients, see the assumption below.
\begin{asm}
    The gradients are upper bounded coordinate-wise: $\|g_k\|_{\infty} \le G_{\infty}$.
\end{asm}

We begin with the analogue of Lemma~\ref{lem:da_s}:
\begin{lem}\label{lem:da_s_coord}
It holds for the iterates of Algorithm~\ref{alg:prodigy_coord}:
\begin{equation*}
    \left\Vert s_{n+1}\right\Vert_1
    \leq 2d_{n+1}\|a_{n+1}\|_1 + \frac{1}{2d_{n+1}}\sum_{k=0}^{n}\lambda_k^2\|g_{k}\|^2_{\mA_k^{-1}}.
\end{equation*}
\end{lem}
\begin{proof}
    As in the proof of Lemma~\ref{lem:da_s}, let us introduce an extra sequence $\overline{d}_n$:
    \[
        \overline{d}_{n+1} 
        \overset{\mathrm{def}}{=} \frac{\left\Vert s_{n+1} \right\Vert_{\mA_{n+1}^{-1}} ^{2}-\sum_{k=0}^{n}\lambda_k^2\left\Vert g_{k}\right\Vert_{\mA_{k}^{-1}} ^{2} }{2\|s_{n+1}\|_1}.
    \]
    The next step is to show that $\hat d_{n+1}\ge \overline{d}_{n+1}$ by comparing the numerators:
    \begin{align*}
        \hat d_{n+1}\|s_{n+1}\|_1
        &= \sum_{k=0}^n \lambda_k\langle g_k,x_0 - x_k\rangle
        = \sum_{k=0}^n \lambda_k\langle g_k,s_k\rangle_{\mA_k^{-1}}
        = \sum_{k=0}^n \langle s_{k+1}-s_k,s_k\rangle_{\mA_k^{-1}} \\
        &= \sum_{k=0}^n \frac{1}{2}\left[\|s_{k+1}\|^2_{\mA_k^{-1}} - \|s_{k+1}-s_k\|^2_{\mA_k^{-1}} - \|s_k\|^2_{\mA_k^{-1}}\right] \\
        &= \frac{1}{2}\|s_{n+1}\|^2_{\mA_n^{-1}}+ \frac{1}{2}\sum_{k=0}^n \|s_{k+1}\|^2_{\mA_k^{-1} - \mA_{k+1}^{-1}}-\frac{1}{2}\sum_{k=0}^n \lambda_k^2\|g_k\|^2_{\mA_k^{-1}}  \\
        &\hspace{-6.1mm}\overset{\mA_k^{-1}\succcurlyeq \mA_{k+1}^{-1}}{\ge} \frac{1}{2}\|s_{n+1}\|^2_{\mA_{n+1}^{-1}} - \frac{1}{2}\sum_{k=0}^n \lambda_k^2\|g_k\|^2_{\mA_k^{-1}} \\
        &= \overline{d}_{n+1} \|s_{n+1}\|_1.
    \end{align*}
    Using the definition of $\overline{d}_{n+1}$,
    and the property $\overline{d}_{n+1}\le \hat{d}_{n+1}\le d_{n+1}$, we derive
    \[
    \frac{1}{2}\left\Vert s_{n+1}\right\Vert_{\mA_{n+1}^{-1}} ^{2}-\frac{1}{2}\sum_{k=0}^{n}\lambda_k^2\left\Vert g_{k} \right\Vert_{\mA_k^{-1}} ^{2} 
    = \overline{d}_{n+1}\left\Vert s_{n+1}\right\Vert_1 
    \le d_{n+1}\left\Vert s_{n+1}\right\Vert_1.
    \]
    Using inequality $2\alpha\beta\le \alpha^2 + \beta^2$ with $\alpha^2 = 2d_{n+1}^2a_{(n+1)i}$ and $\beta^2= \frac{1}{2a_{(n+1)i}}s_{(n+1)i}^2$ for $i=1,\dotsc, p$ and then the bound above, we establish
    \begin{align*}
        2d_{n+1} \|s_{n+1}\|_1
        &= \sum_{i=1}^p d_{n+1}|s_{(n+1)i}|
        \le \sum_{i=1}^p\left(2d_{n+1}^2a_{(n+1)i} + \frac{1}{2a_{(n+1)i}}s_{(n+1)i}^2\right) \\
        &= 2d_{n+1}^2\|a_{n+1}\|_1 + \frac{1}{2}\|s_{n+1} \|_{\mA_{n+1}^{-1}} \\
        &\le 2d_{n+1}^2\|a_{n+1}\|_1 + d_{n+1}\|s_{n+1} \|_{1} + \frac{1}{2}\sum_{k=0}^{n}\lambda_k^2\|g_{k}\|^2_{\mA_k^{-1}}.
    \end{align*}
    Rearranging the terms, we get
    \begin{align*}
        d_{n+1} \|s_{n+1}\|_1
        &\le 2d_{n+1}^2\|a_{n+1}\|_1 + \frac{1}{2}\sum_{k=0}^{n}\lambda_k^2\|g_{k}\|^2_{\mA_k^{-1}}.
    \end{align*}
    It remains to divide both sides by $d_{n+1}$.
\end{proof}
The next lemma is similar to Lemma~\ref{lem:da_s_coord} except that it uses $\ell_{\infty}$ norm for the distance to a solution and $\ell_1$ norm for the weighted gradient sum $s_n$.
\begin{lem}\label{lem:da_sum_f_bounded_by_ds_coord}
    The coordinate-wise version of Prodigy (Algorithm~\ref{alg:prodigy_coord}) satisfies
    \begin{equation}
        \sum_{k=0}^n\lambda_k (f(x_k) - f_*)
        \le (D_{\infty} - \hat d_{n+1})\|s_{n+1}\|_1, \label{eq:da_sum_f_bounded_by_ds_coord}
    \end{equation}
    where $D_{\infty} = \|x_0 - x_*\|_{\infty}$.
\end{lem}
\begin{proof}
    Summing inequality $f(x_k) - f_*\le \langle g_k, x_k - x_*\rangle $ with weights $\lambda_k$, we get
    \begin{align*}
        \sum_{k=0}^n \lambda_k (f(x_k) - f_*) 
        \le \sum_{k=0}^n \lambda_k \langle g_k, x_k - x_*\rangle
        = \sum_{k=0}^n \lambda_k \langle g_k, x_0 - x_*\rangle + \sum_{k=0}^n \lambda_k \langle g_k, x_k - x_0\rangle.
    \end{align*}
    Using H{\"o}lder's inequality on the first product in the right-hand side and then telescoping the second sum, we obtain
    \begin{align*}
        \sum_{k=0}^n \lambda_k (f(x_k) - f_*) 
        &\le \| s_{n+1}\|_1 \|x_0 - x_*\|_{\infty} + \sum_{k=0}^n \lambda_k\langle g_k, x_k - x_0\rangle  \\
        &= \| s_{n+1}\|_1D_{\infty} - \hat d_{n+1}\|s_{n+1}\|.
    \end{align*}
    The use of $\ell_{1}$ norm for the term $s_{n+1}$ above is motivated by the fact that it naturally arises in other parts of the theory.
\end{proof}

\begin{thm}
    Algorithm~\ref{alg:prodigy_coord} converges with the rate
    \[
        f(\overline x_t) - f_*
        \le \frac{4pG_{\infty}D_{\infty}}{\sqrt{n}}\sqrt{\log_{2+}\Bigl(\frac{D_{\infty}}{d_0}\Bigr)},
    \]
    where $t=\arg\min_{k\le n} \frac{d_{k+1}}{\sqrt{\sum_{i=0}^k d_i^2}}$.
\end{thm}
\begin{proof}
    From Lemma~\ref{lem:da_sum_f_bounded_by_ds_coord}, we get
    \begin{align*}
        0
        &\le \sum_{k=0}^n \lambda_k (f(x_k) - f_*) \overset{\eqref{eq:da_sum_f_bounded_by_ds_coord}}{\le} (D_{\infty} - \hat d_{n+1})\| s_{n+1}\|_1,
    \end{align*}
    so we can prove by induction that $d_{n+1}\le D_{\infty}$. Using the same bounds as before, we get for the average iterate
    \begin{align*}
        \sum_{k=0}^n \lambda_k (f(x_k) - f_*)
        &\le D_{\infty}\|s_{n+1}\|_1 - \sum_{k=0}^n \lambda_k\langle g_k, x_0 - x_k\rangle \\
        &= D_{\infty}\| s_{n+1}\|_1 + \frac{1}{2}\sum_{k=0}^n \lambda_k^2\|g_k\|^2_{\mA_k^{-1}} + \frac{1}{2}\sum_{k=0}^n \|s_{k}\|^2_{\mA_k^{-1} - \mA_{k+1}^{-1}} - \frac{1}{2}\|s_{n+1}\|^2_{\mA_{n+1}^{-1}} \\
        &\le D_{\infty}\| s_{n+1}\|_1 + \frac{1}{2}\sum_{k=0}^n \lambda_k^2\|g_k\|^2_{\mA_k^{-1}}.
    \end{align*}
    Let us plug in the bound from Lemma~\ref{lem:da_s_coord}:
    \begin{align*}
        \sum_{k=0}^n \lambda_k (f(x_k) - f_*)
        &\le 2D_{\infty}d_{n+1}\|a_{n+1}\|_1 + \frac{D_{\infty}}{2d_{n+1}}\sum_{k=0}^{n}\lambda_k^2\|g_{k}\|^2_{\mA_k^{-1}} + \frac{1}{2}\sum_{k=0}^n \lambda_k^2\|g_k\|^2_{\mA_k^{-1}} \\
        &\hspace{-5.1mm}\overset{d_{n+1}\le D_{\infty}}{\le} 2D_{\infty}d_{n+1}\|a_{n+1}\|_1 + \frac{D_{\infty}}{d_{n+1}}\sum_{k=0}^{n}\lambda_k^2\|g_{k}\|^2_{\mA_k^{-1}} \\
        &\hspace{-3mm}\overset{\lambda_k\le \lambda_n}{\le} 2D_{\infty}d_{n+1}\|a_{n+1}\|_1 + \frac{D_{\infty}}{d_{n+1}}\lambda_n\sum_{k=0}^{n}\lambda_k\|g_{k}\|^2_{\mA_k^{-1}} .
    \end{align*}
    We now apply Proposition~\ref{pr:da_sequence_bound},  substitute $\lambda_k = d_k^2$, and use $g_{kj}^2\le G_{\infty}^2$:
    \begin{align*}
        \sum_{k=0}^n d_k^2 (f(x_k) - f_*)
        &\le 2D_{\infty}d_{n+1}\|a_{n+1}\|_1 + \frac{D_{\infty}}{d_{n+1}}\lambda_n\sum_{j=1}^p\sum_{k=0}^{n}\frac{\lambda_k g_{kj}^2}{\sqrt{d_k^2G_{\infty}^2 + \sum_{i=0}^{k-1}\lambda_i g_{ij}^2}} \\
        &\le 2D_{\infty}d_{n+1}\|a_{n+1}\|_1 +  \frac{2D_{\infty}}{d_{n+1}} \lambda_n\sum_{j=1}^p\sqrt{ \sum_{k=0}^{n}\lambda_kg_{kj}^2} \\
        &\le 4D_{\infty} d_{n+1}pG_{\infty}\sqrt{\sum_{k=0}^{n}d_k^2}.
    \end{align*}
    Using Lemma~\ref{lem:d_sequence}, we get the rate for $t=\arg\min_{t'\le n}\frac{d_{t'+1}}{\sqrt{\sum_{k=0}^{t'}d_k^2}}$:
    \begin{align*}
        f(\overline x_t) - f_*
        &\le \frac{4pG_{\infty}D_{\infty}}{\sqrt{n}}\sqrt{\log_{2+}\Bigl(\frac{D_{\infty}}{d_0}\Bigr)}.
    \end{align*}
\end{proof}

\section{Lower Complexity Theory}
\label{sec:lowertheory}
\begin{thm}
Consider any exponentially bounded algorithm for minimizing a convex $G$-Lipschitz function
starting from $x_{0}$, which has no knowledge of problem constants G and D. There exists a fixed gradient oracle such that any sequence
of $x_{1,\dots},x_{n}$, there exists a convex Lipschitz problem
$f$ with $G=1$ and $\left\Vert x_{0}-x_{*}\right\Vert \leq D$ for all minimizing points $x_*$, consistent with the gradient oracle such that:
\[
\min_{k\leq n}f(x_{k})-f_{*}\geq\frac{DG\sqrt{\log_{2}(D/x_{1})}}{2\sqrt{n+1}}.
\]
\end{thm}
 
\begin{proof} We consider the construction of a 1D oracle
for this problem. Our oracle returns $g_{0}=-1$ and $f(x_{k})=-x_{k}$
for all queries. Without loss of generality we assume that $x_{k}>0$
for all $k\geq1$, and $G=1$.

For each step $k\geq1$ we define: 
\[
x_{*}=2^{n+1}x_{1},
\]
and thus $D=\left|x_{0}-x_{*}\right|=2^{k+1}x_{1}$. and our construction
uses the following function value and gradient sequence 
\[
f(x)=\left|x-x_{*}\right|+x_{*}.
\]
Note that for all query points $x$, the gradient is negative, and
only the left arm of the absolute value function is seen by the algorithm,
so the function appears linear for all test points. Using this construction,
we have: 
\begin{align*}
\min_{k\leq n}\left[f(x_{k})-f_{*}\right] & =\min_{k\leq n}\left(x_{*}-x_{k}\right)\\
 & =2^{n+1}x_{1}-\max_{k\leq n}x_{k}\\
 & \geq2\cdot2^{n}x_{1}-2^{n}x_{1}\\
 & =2^{n}x_{1}\\
 & =\frac{1}{2}D_{n}.
\end{align*}
Now note that: 
\begin{align*}
\sqrt{\log_{2}(D_{n}/x_{1})} & =\sqrt{\log_{2}(2^{n+1})}\\
 & =\sqrt{n+1}.
\end{align*}
So:
\[
1\geq\frac{\sqrt{\log_{2}(D_{n}/x_{1})}}{\sqrt{n+1}}.
\]
Combining these two results: 
\begin{align*}
\min_{k\leq n}f(x_{k})-f_{*} & \geq\frac{1}{2}D=\frac{1}{2}DG\\
 & =\frac{\frac{1}{2}DG\sqrt{\log_{2}(D/x_{1})}}{\sqrt{n+1}}.
\end{align*}
\end{proof}

\begin{thm}
D-Adaptation, DoG and Prodigy are exponentially bounded.
\end{thm}
\begin{proof}
Consider the $D$ lower bound from D-Adaptation:
\begin{align*}\hat{d}_{n+1}=\frac{\sum_{k=0}^{n}\lambda_{k}\gamma_{k}\left\langle g_{k},s_{k}\right\rangle }{\|s_{n+1}\|},\end{align*}
with:
\[
s_{n+1}=\sum_{k=0}^{n}d_{k}g_{k}.
\]
Recall that 
\[
\sum_{k=0}^{n}\lambda_{k}\gamma_{k}\left\langle g_{k},s_{k}\right\rangle \leq\gamma_{n+1}\left\Vert s_{n+1}\right\Vert ^{2}.
\]
Note also that $\gamma_{n+1}\leq\frac{1}{G}$. So:
\begin{align*}
d_{n+1} & \leq\frac{\frac{1}{G}\left\Vert s_{n+1}\right\Vert ^{2}}{\left\Vert s_{n+1}\right\Vert } \leq\frac{1}{G}\left\Vert \sum_{k=0}^{n}d_{k}g_{k}\right\Vert \leq\sum_{k=0}^{n}d_{k}.
\end{align*}
So the sequence $d_{n}$ is upper bounded by the sequence:
\[
a_{n}=\begin{cases}
\sum_{k=0}^{n-1}a_{k} & n\geq1\\
d_{0} & n=0
\end{cases}.
\]
This sequence has the following closed form:
\[
a_{n+1}=2^{n}d_{0}\quad \text{for\ }n\geq 1.
\]
We can prove this by induction. The base case is by definition $a_{1}=a_{0}$.
Then 
\begin{align*}
a_{n+1} & =\sum_{k=0}^{n}a_{k}=\sum_{k=0}^{n-1}a_{k}+a_{n}
 =a_{n}+a_{n}
 =2a_{n}
 =2^{n}d_{0}.
\end{align*}
Note that for both the Dual Averaging form and the GD form we have, we have:
\[
\left\Vert x_{n+1}-x_{0}\right\Vert \leq \left\Vert \frac{1}{G}\sum_{k=0}^{n}d_{k}g_{k}\right\Vert \leq\sum_{k=0}^{n}d_{k}\leq d_{n+1}\leq2^{n}d_{0}.
\]
It follows that D-Adaptation is exponentially bounded. 
For Prodigy, note that:

\[
\gamma_{n+1}\leq\frac{1}{\sqrt{d_{n+1}^{2}G^{2}}}
=\frac{1}{d_{n+1}G}.
\]
Therefore
\begin{align*}
d_{n+1} & \leq\frac{\frac{1}{d_{n+1}G}\left\Vert s_{n+1}\right\Vert ^{2}}{\left\Vert s_{n+1}\right\Vert }\leq\frac{1}{d_{n+1}G}\left\Vert \sum_{k=0}^{n}d_{k}^{2}g_{k}\right\Vert \\
 & \leq\frac{1}{d_{n+1}}\sum_{k=0}^{n}d_{k}^{2}\\
 & \leq\frac{1}{d_{n+1}}\sum_{k=0}^{n}d_{k}d_{n+1}\\
 & \le\sum_{k=0}^{n}d_{k}.
\end{align*}
The rest of the argument follows the D-Adaptation case, with:
\[
\left\Vert x_{n+1}-x_{0}\right\Vert \leq \left\Vert \frac{1}{d_{n}G}\sum_{k=0}^{n}d_{k}^{2}g_{k}\right\Vert \leq\sum_{k=0}^{n}d_{k}\leq d_{n+1}\leq2^{n}d_{0}.
\]

For DoG, recall the basic DoG step is gradient descent with step sizes:
\[
\gamma_{k}=\frac{\bar{r}_{k}}{\sqrt{G^{2}+\sum_{i=0}^{k}\left\Vert g_{i}\right\Vert ^{2}}}.
\]

Using the triangle inequality we have:
\begin{align*}
\left\Vert x_{k+1}-x_{0}\right\Vert  & =\left\Vert x_{k}-\gamma_{k}g_{k}-x_{0}\right\Vert \\
 & \leq\left\Vert x_{k}-x_{0}\right\Vert +\text{\ensuremath{\gamma_{k}}\ensuremath{\left\Vert g_{k}\right\Vert }}\\
 & \leq\left\Vert x_{k}-x_{0}\right\Vert +\frac{\bar{r}_{k}}{\sqrt{G^{2}}}\ensuremath{\left\Vert g_{k}\right\Vert }\\
 & \leq\left\Vert x_{k}-x_{0}\right\Vert +\bar{r}_{k}\\
 & \leq 2\bar{r}_{k}.
\end{align*}

Chaining gives the result.

\end{proof}

\begin{prop}
Suppose that $d_{k}\le cD$ and $\gamma_{k}\leq d_{k}/G$. then:
\[
\left\Vert x_{k}-x_{0}\right\Vert \leq\left(2c+1\right)^{n}\left\Vert x_{1}-x_{0}\right\Vert .
\]
\end{prop}

\begin{proof}
Without loss of generality assume that $G=1$. Firstly, note that using the absolute value function as constructed
in Theorem \ref{thm:explb}, it's clear that there is always exists a function with
$D_{k}\leq2\left\Vert x_{k}-x_{*}\right\Vert $ at step $k$ consistent
with the sequence of gradients seen so far. Therefore, it must hold
that 
\[
d_{k}\leq cD_{k}\leq2c\left\Vert x_{k}-x_{0}\right\Vert .
\]

We prove the result by induction. For the base case, trivially:
\[
\left\Vert x_{1}-x_{0}\right\Vert \leq\left(2c+1\right)^{1}\left\Vert x_{1}-x_{0}\right\Vert .
\]

For the inductive case:
\begin{align*}
\left\Vert x_{k+1}-x_{0}\right\Vert  & =\left\Vert x_{k}-\gamma_{k}g_{k}-x_{0}\right\Vert \\
 & \leq\left\Vert x_{k}-x_{0}\right\Vert +\text{\ensuremath{\gamma_{k}}\ensuremath{\left\Vert g_{k}\right\Vert }}\\
 & \leq\left\Vert x_{k}-x_{0}\right\Vert +\frac{cD_{k}}{G}\ensuremath{\left\Vert g_{k}\right\Vert }\\
 & \leq\left\Vert x_{k}-x_{0}\right\Vert +cD_{k}\\
 & \leq\left(2c+1\right)\left\Vert x_{k}-x_{0}\right\Vert \\
 & \leq\left(2c+1\right)^{n+1}\left\Vert x_{1}-x_{0}\right\Vert .
\end{align*}
\end{proof}

\end{document}